\documentclass{article}

\usepackage{microtype}
\usepackage{graphicx}
\usepackage{subfigure}
\usepackage{booktabs} 

\usepackage[colorlinks,linkcolor={blue!50!black},citecolor={blue!50!black},urlcolor={blue!50!black}]{hyperref}

\usepackage[accepted]{icml2024}

\usepackage{amsmath}
\usepackage{amssymb}
\usepackage{mathtools}
\usepackage{amsthm}

\usepackage{amsmath,amsfonts,bm,amssymb,mathtools}

\def\eqref#1{equation~\ref{#1}}

\def\1{\bm{1}}

\def\vd{{\bm{d}}}

\def\vf{{\bm{f}}}

\def\vm{{\bm{m}}}

\def\vs{{\bm{s}}}

\def\vw{{\bm{w}}}
\def\vx{{\bm{x}}}
\def\vy{{\bm{y}}}
\def\vz{{\bm{z}}}

\def\mI{{\bm{I}}}

\def\mM{{\bm{M}}}

\DeclareMathAlphabet{\mathsfit}{\encodingdefault}{\sfdefault}{m}{sl}
\SetMathAlphabet{\mathsfit}{bold}{\encodingdefault}{\sfdefault}{bx}{n}

\def\gG{{\mathcal{G}}}

\def\gI{{\mathcal{I}}}

\def\gN{{\mathcal{N}}}
\def\gO{{\mathcal{O}}}
\def\gP{{\mathcal{P}}}

\newcommand{\E}{\mathbb{E}}

\newcommand{\KL}{D_{\mathrm{KL}}}

\renewcommand{\mid}{\,|\,}
\newcommand{\GP}{\gG\gP}
\newcommand{\vS}{\bm{\Sigma}}

\usepackage[capitalize,noabbrev,nameinlink]{cleveref}

\theoremstyle{plain}
\newtheorem{theorem}{Theorem}[section]

\theoremstyle{definition}

\theoremstyle{remark}

\usepackage{multirow}
\usepackage{tabularx}
\usepackage{booktabs}
\usepackage{stfloats}

\newlength\savewidth\newcommand\shline{\noalign{\global\savewidth\arrayrulewidth
  \global\arrayrulewidth 1pt}\hline\noalign{\global\arrayrulewidth\savewidth}}

\usepackage[disable,textsize=tiny]{todonotes}

\icmltitlerunning{Differentiable AIS Minimizes The
Symmetrized KL Divergence Between Initial and Target Distribution}

\begin{document}

\twocolumn[
\icmltitle{
Differentiable Annealed Importance Sampling Minimizes The Symmetrized\\
Kullback-Leibler Divergence Between Initial and Target Distribution
}

\begin{icmlauthorlist}
\icmlauthor{Johannes Zenn}{TAIC,UT,IMPRS}
\icmlauthor{Robert Bamler}{TAIC,UT}
\end{icmlauthorlist}

\icmlaffiliation{UT}{University of Tübingen}
\icmlaffiliation{TAIC}{Tübingen AI Center}
\icmlaffiliation{IMPRS}{IMPRS-IS}

\icmlcorrespondingauthor{Johannes Zenn}{johannes.zenn@uni-tuebingen.de}

\icmlkeywords{Machine Learning, probabilistic methods, approximate Bayesian inference, annealed importance sampling, divergences, variational, inference, ICML}

\vskip 0.3in
]



\printAffiliationsAndNotice{}  

\begin{abstract}
Differentiable annealed importance sampling (DAIS), proposed by \citet{geffner2021mcmc} and \citet{zhang2021differentiable}, allows optimizing over the initial distribution of AIS.
In this paper, we show that, in the limit of many transitions, DAIS minimizes the symmetrized Kullback-Leibler divergence between the initial and target distribution.
Thus, DAIS can be seen as a form of variational inference (VI) as its initial distribution is a parametric fit to an intractable target distribution.
We empirically evaluate the usefulness of the initial distribution as a variational distribution on synthetic and real-world data, observing that it often provides more accurate uncertainty estimates than VI (optimizing the reverse KL divergence), importance weighted VI, and Markovian score climbing (optimizing the forward KL divergence).
\end{abstract}

\section{Introduction}

Annealed importance sampling~(AIS) \citep{neal2001annealed} allows estimating the normalization constant $Z:={\int\! f(\vz)\,\mathrm{d}\vz}$ of an unnormalized distribution~$f(\vz)$.
In probabilistic machine learning, AIS can be used for Bayesian model selection \citep{knuth2015bayesian}, where the goal is to
maximize the marginal likelihood $p(\vx) = {\int\! p(\vz,\vx)\,\mathrm{d}\vz}$ over a family of candidate probabilistic models~$p$.
Here, $p(\vz,\vx)$ is the model's joint distribution over latent variables~$\vz$ and (fixed) observed data~$\vx$.
In this paper, we investigate the initial distribution of AIS or, more specifically, \emph{differentiable} AIS (DAIS) \citep{geffner2021mcmc,zhang2021differentiable}.
DAIS combines aspects of Markov Chain Monte Carlo (MCMC) and Variational Inference (VI):
it draws samples from a (variational) initial distribution~$q_0$ and then follows MCMC dynamics to move the samples towards a target distribution.

We show theoretically that, in the limit of many transitions, DAIS minimizes the symmetrized Kullback-Leibler (KL) divergence\footnote{
In a previous version, we used the term Jensen-Shannon divergence.
To be precise, we have updated our terminology.
} \citep{kullback1951information} between its initial and target distribution.
The symmetrized KL divergence is the sum of the reverse and the forward KL divergence.
While the reverse KL divergence, used in VI, is known to be \textit{mode-seeking}, the forward KL divergence, used in Markovian Score Climbing (MSC) \citep{naesseth2020markovian}, is associated with a \textit{mass-covering} behavior \citep{pml2Book}. The symmetrized KL divergence averages between both.

We empirically analyze implications of this theoretical result by asking: ``How useful is the initial distribution~$q_0$ of DAIS as a parametric approximation of the target distribution?''
This question corresponds to an inference scheme (which we denote by $\text{DAIS}_0$) that is identical to DAIS at training time (i.e., it moves particles sampled from~$q_0$ along an MCMC chain).
At inference time, however, $\text{DAIS}_0$ only uses~$q_0$, omitting expensive AIS steps.
Here, we refer to fitting the model parameters as ``training'' and running the model on test data as ``inference''.
Generally, we do not expect $\text{DAIS}_0$ to match the performance of DAIS.
However, $\text{DAIS}_0$ has the advantage of providing an \textit{explicit} and \textit{compact} representation of the approximate target distribution.

We call a distribution \textit{explicit} if it has an analytic expression (as opposed to, e.g., an algorithmic prescription for sampling from it).
It is often stated that an explicit expression for the exact Bayesian posterior $p(\vz\mid\vx) = p(\vz,\vx)/p(\vx)$ is intractable in large models, since calculating $p(\vx)$ is prohibitively expensive.
However, this adage misses an important point.
Even if we had an oracle that told us the value of $p(\vx)$, the explicit expression for the exact posterior would typically be far too complicated to be of any practical use in downstream tasks because it would, in general, have one term per data point.
We argue that we are actually interested in a \textit{compact} approximation of the posterior distribution, i.e., a tractable parameterized distribution that one can efficiently evaluate and sample from.
Having a compact approximate posterior enables various downstream applications, such as continual learning \citep{nguyen2018variational}, pruning in BNNs \citep{xiao2023compact}, and compression of neural networks \citep{tan2022posttraining} or other data \citep{yang2020variational}.

Starting from our theoretical contribution---showing that DAIS minimizes the symmetrized KL divergence between $q_0$ and ${f\,/\,Z}$---we provide an empirical analysis of $\text{DAIS}_0$ for approximate Bayesian inference.
We compare $\text{DAIS}_0$, DAIS, VI, importance weighted VI (IWVI), and MSC.
We find that $\text{DAIS}_0$ often provides more accurate uncertainty estimates than VI, IWVI, and MSC on Gaussian process regression and real-world datasets for logistic regression.

\section{Related Work} \label{sec:related-work}

In the following, we discuss related work on variational distributions that are augmented by MCMC transitions, the forward and reverse KL divergence, and DAIS.
We cover related work on methods to estimate normalization constants in \Cref{sec:background}.
Related work on bounds with respect to the symmetrized KL divergence is given in  \Cref{sec:method}.

\paragraph{MCMC-Augmented Variational Distributions.}
Sequential Monte Carlo samplers (SMCS) \citep{del2006sequential} are methods derived from prarticle filters \citep{doucet2001introduction} to estimate normalization constants. 
While SMCS and annealed importance sampling (AIS) both describe the same mathematical framework, SMCS typically integrate a resampling step that lets the model focus on ``important'' particles. 
AIS can be seen as a finite-difference approximation to thermodynamic integration (TI) \citep{gelman1998simulating} which computes ratios of normalization constants by solving a one-dimensional integration problem.
\citet{masrani2019thermodynamic} connect TI and variational inference (VI) resulting in tighter variational bounds by using importance sampling to compute the integral.
\citet{thin2021monte} propose a variant of the algorithm based on sequential importance sampling.
Hamiltonian variational inference combines variational inference and MCMC iterations differentiably by augmenting the variational distribution with MCMC steps \citep{salimans2015markov,wolf2016variational}.
The Hamiltonian VAE \citep{caterini2018hamiltonian} builds on Hamiltonian Importance Sampling \citep{neal2005hamiltonian} and improves on HVI by using optimally chosen reverse MCMC kernels.
$\text{DAIS}_0$, investigated here, uses an MCMC-augmented variational distribution during training but not during inference.

\paragraph{VI With Forward and Reverse KL Divergence.}
Black box VI \citep{ranganath2014black} is typically associated with the reverse KL divergence (between the variational distribution and the real posterior distribution).
However, also various other divergences have been investigated to obtain an approximation of the posterior distribution \citep{hernandez2016black,li2016renyi,dieng2017variational,wan2020f}.
Most related to $\text{DAIS}_0$, \citet{ruiz2019contrastive} refine the variational distribution by running MCMC transitions.
They minimize a contrastive divergence which, in the limit, converges to the symmetrized KL divergence.
Markovian score climbing (MSC) \citep{naesseth2020markovian}, that we also compare to in \Cref{sec:experiments}, optimizes the forward KL divergence using unbiased stochastic gradients.
MSC samples a Markov chain and uses the samples to follow the score of the variational distribution.
The Markov kernel for the MCMC dynamics is based on importance sampling.

\paragraph{Differentiable Annealed Importance Sampling (DAIS).}
\citet{geffner2021mcmc} and \citet{zhang2021differentiable} propose DAIS concurrently focusing on different aspects of the model (empirical results and a convergence analysis for linear regression, respectively).
\citet{doucet2022scorebased} identify that the backward transitions of AIS are conveniently chosen but suboptimal.
They propose Monte Carlo diffusion (MCD) which parameterizes the time reversal of the forward dynamics and which can be learned by maximizing a lower bound on the marginal likelihood, or equivalently, a denoising score matching loss.
\citet{geffner2023langevin} provide a more general framework for MCD and investigate various dynamics and numerical simulation schemes.
\citet{zenn2023resampling} extend DAIS to a sequential Monte Carlo sampler and provide a theoretical argument for leaving out the gradients corresponding to the resampling operation.

\section{Estimating Normalization Constants} \label{sec:background}

\begin{figure}[t]
    \centering
    \includegraphics[width=0.7\linewidth]{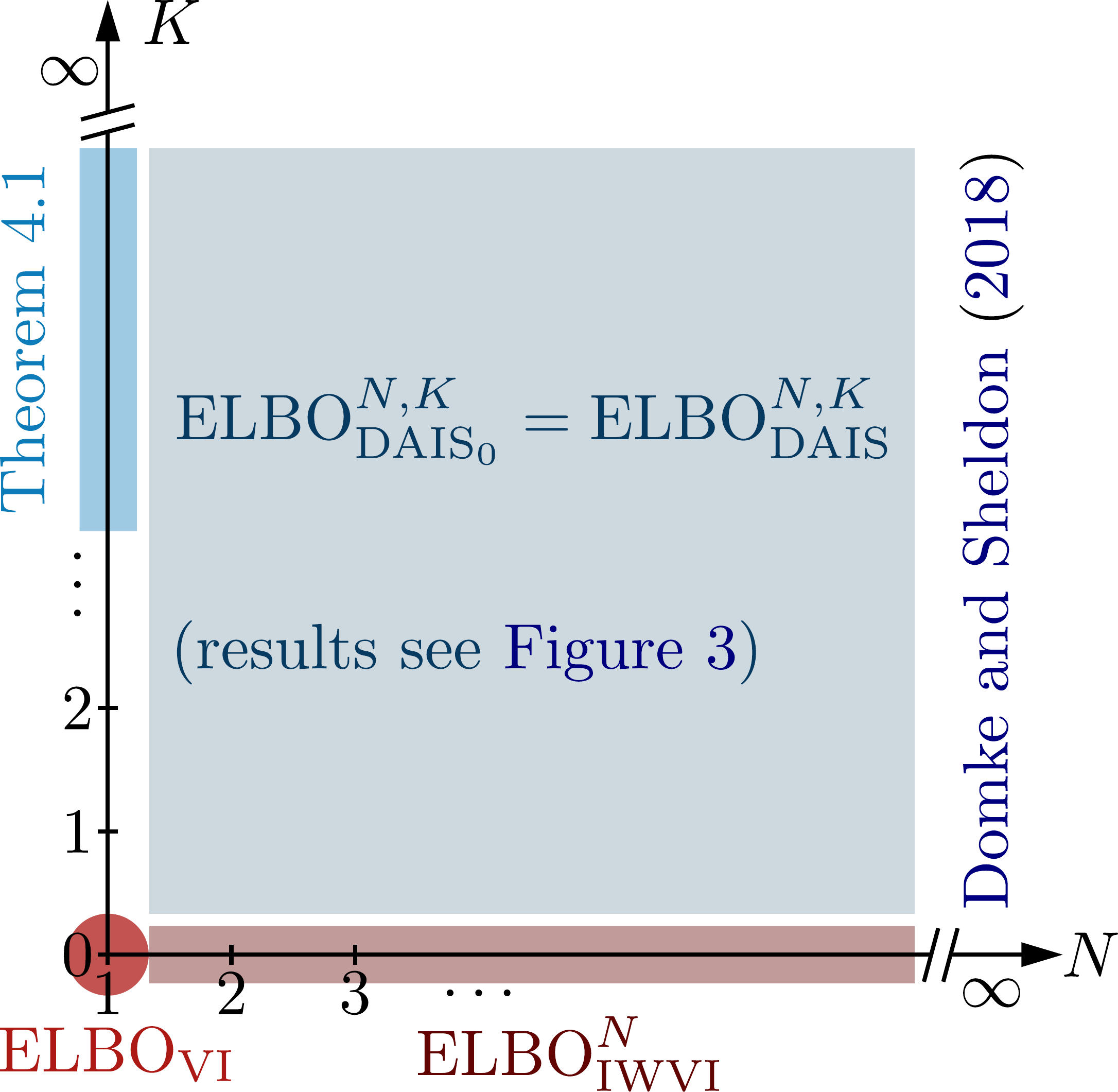}
    \caption{
    The landscape spanned by various lower bounds to the normalization constant $Z={\int\!f(\vz)\,\mathrm{d}\vz}$ where $N$ denotes number of particles and $K$ the number of importance sampling transitions.
    A discussion and further details can be found in \Cref{sec:background}.
    }
    \label{fig:n-k-d-overview}
\end{figure}

In this section, we discuss how variational inference (VI), importance weighted VI (IWVI), and (differentiable) AIS ((D)AIS) can be understood as approaches to reducing the variance of importance sampling (IS).
\Cref{fig:n-k-d-overview} summarizes the space spanned by these methods, highlights a limiting behavior \citep{domke2018importance}, and shows how our theoretical result (\Cref{thm:main}) fits into this unifying framework. 
\Cref{sec:vi} and \Cref{sec:iw-elbo} give an overview over (IW)VI.
\Cref{sec:dais} then introduces our notation for (D)AIS.

\paragraph{Importance Sampling} is a principal technique for estimating the integral over a (nonnegative) function~$f$ by sampling from a normalized proposal distribution~$q$,
\begin{align}\label{eq:is}
    Z \coloneqq
    \int\! f(\vz)\,\text{d}\vz = 
    \E_{\vz \sim q}
    \left[
        \frac{f(\vz)}{q(\vz)}
    \right],
\end{align}
where we assume that the support of~$q$ contains the support of~$f$.
We need to be able to efficiently draw samples (or \emph{particles}) $\vz$ from~$q$, and to evaluate $f(\vz)$ and $q(\vz)$ for these samples.
How many samples are necessary to obtain an accurate estimate of~$Z$ depends on how well $q$ approximates ${f\,/\,Z}$.
In high dimensions, the mismatch between~${f\,/\,Z}$ and any simple proposal distribution~$q$ grows, which leads to high variance of the importance weight $w(\vz) \coloneqq f(\vz)\,/\,q(\vz)$, and the number of samples has to grow exponentially in the dimension \citep{agapiou2017importance}.
We now discuss how VI, IWVI, and AIS all aim to reduce this variance.

\subsection{Variational Inference} \label{sec:vi}

Variational inference (VI) applies the logarithm to the importance weights $w(\vz) \coloneqq f(\vz)\,/\,q(\vz)$ in \Cref{eq:is}.
This reduces exponentially growing variances to linearly growing ones, but it introduces bias, resulting in a lower bound on $\log Z$ by Jensen's inequality \citep{blei2017variational},
\begin{align} \label{eq:elbo}
    \begin{split}
    \log Z
    &\geq
    \E_{\vz \sim q}
    \left[
        \log\left(
            \frac{f(\vz)}{q(\vz)}
        \right)
    \right]
    \eqqcolon
    \mathrm{ELBO}_{\text{VI}}(f, q).
    \end{split}
\end{align}

\Cref{eq:elbo} is called the evidence lower bound (ELBO) since VI is often used to estimate the evidence $\log p(\vx) = {\log\int\! p(\vz,\vx)\,\mathrm{d}\vz}$ of a probabilistic model $p(\vz, \vx)$ (latent variables~$\vz$ and observed data~$\vx$).
VI maximizes the ELBO over parameters of~$q$, leading to the best approximation of $\log p(\vx)$, which is useful for approximate Bayesian model selection \citep{beal2003variational,kingma2013auto}.
In addition, VI is a popular method for approximate Bayesian inference as the distribution~$q$ that maximizes the ELBO also approximates the (intractable) posterior $p(\vz\mid\vx) = p(\vz,\vx)\,/\,Z$.
This is because maximizing the ELBO over parameters of~$q$ is equivalent to minimizing the gap $\Delta_\text{VI} := \log Z - \mathrm{ELBO}_\text{VI}(f,q) \geq 0$, which turns out to be the KL divergence from the true posterior to~$q$, i.e., $\Delta_\text{VI} = \KL(q(\vz)\,\|\, p(\vz\mid\vx))$ \citep{blei2017variational}.

As we discuss next, IWVI and AIS can both be seen as methods that reduce the gap of VI by reducing the sampling variance.
AIS is typically considered a sampling method, prioritizing the quality of samples or an accurate estimate of~$Z$ over a good approximate posterior distribution~$q$, which would be the perspective of VI.
In \Cref{sec:method}, we analyze (differentiable) AIS from the perspective of VI.

\subsection{Importance Weighted Variational Inference} \label{sec:iw-elbo}

Importance weighted variational inference (IWVI) \citep{burda2016importance,domke2018importance} reduces the variance of the importance weight $w(\vz)$ by averaging $N$~independent samples from~$q$, where $N\geq 2$, i.e., it uses the weights
\begin{align}\label{eq:weight-IWVI}
    w_\text{IWVI}^N\bigl(\vz^{(1:N)}\bigr)
    \coloneqq
    \frac1N \sum_{i=1}^N \frac{f(\vz^{(i)})}{q(\vz^{(i)})}
\end{align}
where ${\vz^{(i)}\sim q}$ for all~$i$.
The resulting bound,
\begin{align} \label{eq:iw-elbo}
    \begin{split}
    \log Z &\geq 
    \E_{\vz^{(1:N)} \sim q}
    \left[
        \log
        \left(
            w_\text{IWVI}^N\bigl(\vz^{(1:N)}\bigr)
        \right)
    \right]\\
    &\eqqcolon
    \textrm{ELBO}_\text{IWVI}^N(f, q),
    \end{split}
\end{align}
recovers $\mathrm{ELBO}_\text{VI}(f,q)$ for $N=1$ (see \Cref{eq:elbo} and red highlights on the $x$ axis of \Cref{fig:n-k-d-overview}).
Sampling from~$q$ at inference time requires a sampling-importance-resampling (SIR) procedure \citep{cremer2017reinterpreting} which we denote by $\text{IWVI}_\text{SIR}$.
The corresponding Markov kernel is known as i-SIR and studied in detail by \citet{andrieu2018uniform}.
As $N$ grows, the bound becomes tighter.
In the limit of $N\to\infty$, \citet{domke2018importance} find (based on results due to \citet[Proposition 1]{maddison2017filtering}) the following.

\begin{theorem}[Theorem 3 in \citet{domke2018importance}] \label{thm:domke-sheldon}
    For large $N$, the gap $\Delta_\textup{IWVI}^N := \log Z - \textup{ELBO}_\textup{IWVI}^N$ of importance weighted VI is proportional to the variance of $w_\textup{IWVI}^N$, defined in \Cref{eq:weight-IWVI}.
    Formally, if $\lim\sup_{N \to \infty}\E_q[1 \,/\, w_\textup{IWVI}^N] < \infty$ and there exists some $\alpha > 0$ such that ${\E_q\bigl[{|w_\textup{IWVI}^N - \log Z|^{2 + \alpha}}\bigr] < \infty}$, then 
    \begin{align}
        \lim_{N \to \infty} 
        N \Delta_\textup{IWVI}^N
        =
        \frac{\operatorname{Var}_{q}\bigl[w_\textup{IWVI}^N\bigl(\vz^{(1:N)}\bigr)\bigr]}{2Z^2}.
    \end{align}
\end{theorem}
\begin{proof}
    See \citet[Theorem 3]{domke2018importance}.
\end{proof}

Thus, for large~$N$, a higher variance of $w_\textup{IWVI}^N$ corresponds to a larger gap (see \Cref{sec:vi}).
Annealed importance sampling, discussed next, provides a way to further reduce the sampling variance for a fixed number of particles~$N$.

\subsection{(Differentiable) Annealed Importance Sampling} \label{sec:dais}

Annealed importance Sampling (AIS) reduces the variance of $w_\text{IWVI}^N(\vz^{(1:N)})$ further by recalling that the variance of the importance weights  $f(\vz^{(i)})\,/\,q(\vz^{(i)})$ in \Cref{eq:weight-IWVI} is a consequence of a distribution mismatch between $q$ and~$f/Z$.
To reduce the distribution mismatch, AIS interpolates between $q$ and ${f\,/\,Z}$ with distributions $\{\pi_{\beta_k}(\vz_k)\}_{k=0}^K$ over auxiliary variables $\vz_0,\ldots,\vz_K$ for $K$ interpolation steps.

In detail, AIS estimates the normalization constant of an unnormalized distribution $f_\text{AIS}$ over all $\vz_0,\ldots,\vz_K$,
\begin{align}
    f_\text{AIS}(\vz_{0:K}) &\coloneqq
    f(\vz_K)
    \prod_{k=1}^K B_k(\vz_{k-1} \mid \vz_k).
    \label{eq:sub:introduce-B}
\end{align}
Here, the so-called backward transition kernels~$B_k$ have to be normalized in their first argument so that $f_\text{AIS}$ and $f$ have the same normalization constant, ${\int\! f_\text{AIS}(\vz_{0:K})\,\mathrm{d}\vz_{0:K}} = {\int\!f(\vz_K)\,\mathrm{d}\vz_K}=Z$.
Originally, \citet{neal2001annealed} estimates~$Z$ with IS, using a proposal distribution of the form
\begin{align}\label{eq:q-ais}
    q_\text{AIS}(\vz_{0:K}) 
    &\coloneqq 
    q_0(\vz_0)\prod_{k=1}^K 
    F_k(\vz_k \mid \vz_{k-1}),
\end{align}
where we call~$q_0$ the initial distribution and $F_k$ a forward transition kernel.
Instead of using $q_\text{AIS}$ with IS to estimate~$Z$, we can use $q_\text{AIS}$ also with IWVI to obtain a bound on~$Z$,
\begin{align}
    \begin{split}
    \log Z &\geq
    \E_{\vz_{0:K}^{(1:N)} \sim q_\text{AIS}}
    \left[
            \log
                \left(
                w_\text{AIS}\bigl(\vz_{0:K}^{(1:N)}\bigr)
            \right)
    \right] 
    \\
    &\eqqcolon
    \text{ELBO}_\text{AIS}^{N,K}(f_\text{AIS}, q_\text{AIS}) \label{eq:elbo-dais}
    \end{split}
\end{align}
with the annealed importance weights
\begin{align}\label{eq:weight-ais}
    \!\!\!\! w_\text{AIS}^{N,K}\!\bigl(\vz_{0:K}^{(1:N)}\bigr)
    = \frac1N\! \sum_{i=1}^N\!
        \frac{f(\vz_K^{(i)})}{q_0(\vz_0^{(i)})}
        \prod_{k=1}^{K}
        \frac{B_k(\vz_{k-1}^{(i)} \mid \vz_{k}^{(i)})}{F_k(\vz_{k}^{(i)} \mid \vz_{k-1}^{(i)})}.\!
\end{align}

\Cref{eq:elbo-dais} and \Cref{eq:weight-ais} hold for arbitrary normalized forward and backward kernels $F_k$ and~$B_k$ (as long as $\operatorname{support}(q_\text{AIS}) \supseteq \operatorname{support}(f_\text{AIS})$).
In practice, however, we want to draw samples~$\vz_k^{(i)}$ from distributions that interpolate smoothly between $q_0(\vz_0)$ and $f(\vz_K)$ so that none of the factors on the right-hand side of \Cref{eq:weight-ais} has a large variance.
One typically realizes each $F_k({\vz_k\mid \vz_{k-1}})$ as a Markov Chain Monte Carlo process (typically Hamiltonian Monte Carlo (HMC)), which leaves a distribution $\pi_{\beta_k}(\vz_k)$ invariant, where the distributions $\{\pi_{\beta_k}\}_{k=0}^K$ interpolate between $\pi_{\beta_0} \coloneqq q_0$ and $\pi_{\beta_K} \coloneqq {f\,/\,Z}$.
The most common choice of~$\pi_{\beta_k}$ uses the geometric mean (aka annealing), i.e.,
\begin{align} \label{eq:def-pi}
    \!\!\!\!\pi_{\beta_k}(\vz) \!\coloneqq\! \frac{\gamma_{\beta_k}(\vz)}{Z_{\beta_k}}
    \quad\!\!\!\text{with}\quad\!\!\!
    \gamma_{\beta_k}(\vz) \!\coloneqq\! q_0(\vz) \! \left(\;\!\!\frac{f(\vz)}{q_0(\vz)}\;\!\!\right)^{\!\!\beta_k}
\end{align}
where $Z_{\beta_k} = {\int\!\gamma_{\beta_k}(\vz)\,\mathrm{d}\vz}$ and $0=\beta_0 <\beta_1 < \cdots < \beta_K=1$.
To ensure that \Cref{eq:weight-ais} can be evaluated for samples $\vz_{1:K} \sim q_\text{AIS}$, one typically sets~$B_k$ to the reversal of~$F_k$, i.e., $B_k({\vz_{k-1} \mid \vz_k}) \coloneqq F({\vz_k \mid \vz_{k-1}})\,\gamma_{\beta_k}(\vz_{k-1})\,/\,\gamma_{\beta_k}(\vz_k)$, which is properly normalized.
With this choice, \Cref{eq:weight-ais} simplifies to
\begin{align}\label{eq:weight-ais-simplified}
    w_\text{AIS}^{N,K}\!\bigl(\vz_{0:K}^{(1:N)}\bigr)
    = \frac1N \sum_{i=1}^N \prod_{k=1}^K
        \frac{\gamma_{\beta_k}\bigl(\vz_{k-1}^{(i)}\bigr)}{\gamma_{\beta_{k-1}}\bigl(\vz_{k-1}^{(i)}\bigr)}.
\end{align}

\paragraph{Differentiable Annealed Importance Sampling.}
To ensure that $F_k({\vz_k\mid \vz_{k-1}})$ leaves $\pi_{\beta_k}$ invariant, the HMC implementation of~$F_k$ involves a Metropolis-Hastings (MH) acceptance or rejection step, which makes the method non-differentiable.
Differentiable annealed importance sampling (DAIS) drops this MH step.
The resulting ``uncorrected'' transition kernels~$F_k$ do not leave~$\pi_{\beta_k}$ invariant.
Instead, the backward kernels $B_k$ are defined by starting from an exact sample and reversing the forward chain.
With these modifications, the resulting lower bound $\text{ELBO}_\text{DAIS}^{N,K}(f_\text{DAIS}, q_\text{DAIS})$ has the same functional form as $\text{ELBO}_\text{AIS}^{N,K}(f_\text{AIS}, q_\text{AIS})$, where only the kernels $F_k$ and $B_k$ differ.
Furthermore, it can be optimized with reparameterization gradients.

\Cref{thm:domke-sheldon} \citep[Theorem 3]{domke2018importance} also applies to this bound.
Therefore, for large $N$, the gap of AIS is proportional to the variance of $w_\text{AIS}^{N,K}$, which, for good choices of the forward and backward kernels $F_k$ and~$B_k$, is smaller than the variance of the estimator of IWVI.

\section{Analyzing the Initial Distribution of DAIS} \label{sec:method}

In \Cref{thm:main} of this section we present our main theoretical result that DAIS minimizes the symmetrized KL divergence between its initial distribution and its target distribution.
Then, we discuss compact representations in terms of their sampling complexity.
Finally, we frame finding the initial distribution of DAIS as a form of VI.

\subsection{DAIS Minimizes the Symmetrized KL Divergence} \label{sec:method:theory}

\Cref{thm:main} below presents our main theoretical contribution showing that DAIS minimizes the symmetrized KL divergence between its initial and target distribution.
We see \Cref{thm:main} as starting point to motivate $\text{DAIS}_0$ and the analysis of the shape of the fitted initial distribution~$q_0$.
The result also holds true for AIS, but it is most relevant for DAIS, which learns parameters of its initial distribution.
One can also show \Cref{thm:main} by combining results of \citet{grosse2013annealing} and \citet{crooks2007measuring} from the physics literature. 
\citet{brekelmans2022improving} state a related bound on the difference between an upper and a lower bound on the evidence (whereas \Cref{thm:main} makes an asymptotic statement about the gap between lower bound and evidence).

\begin{theorem} \label{thm:main}
    Let $\operatorname{support}(q_0) \supset \operatorname{support}(f)$.
    We assume that the transitions between consecutive annealing distributions are perfect and that~$\beta_k$ are equally spaced, i.e., $\beta_k = k/K$.
    Then, for large~$K$ and $N=1$, the gap $\Delta_\mathrm{AIS}^{1,K} \coloneqq \log Z - \mathrm{ELBO}_\mathrm{AIS}^{1,K}(f_\mathrm{AIS}, q_\mathrm{AIS})$ is a divergence between the initial distribution $q_0$ and the target distribution $f/Z$.
    Up to corrections of $\gO(1\,/\,K^3)$, the divergence is proportional to the symmetrized KL divergence,
    \begin{align}
        \begin{split}
        \Delta_\mathrm{AIS}^{1,K}
        =
        \frac{1}{K}
        &\left(
            \frac{1}{2}\KL(f(\vz)\,/\,Z \,\|\, q_0(\vz))
        \right.
        \\
        &\quad
        \left.
        +
        \frac{1}{2}\KL(q_0(\vz) \,\|\, f(\vz)\,/\,Z)
        \right)
        \\
        &
        + \gO\bigl(1\,/\,K^3\bigr).
        \end{split} \label{eq:thm-main}
    \end{align}
\end{theorem}
\vspace{-1em}
\begin{proof}
    We prove the theorem in \Cref{sup:sec:thm-main-proof}.
\end{proof}

\paragraph{Statement for $N>1$.}
For a general $N>1$ we find that $\Delta_\text{AIS}^{N,K} \leq \Delta_\text{AIS}^{1,K}$ and therefore 
\begin{align}
    \lim_{K\to\infty} \Delta_\text{AIS}^{N,K}
    \leq 
    \frac 1K D_\text{JS}(q_0, f/Z) + \gO\bigl(1\,/\,K^3\bigr),
\end{align}
where $D_\text{JS}$ denotes the symmetrized KL divergence.
This shows that the bound for $N>1$ can be upper-bounded by the symmetrized KL divergence but does not give additional insights on, e.g., symmetry. 
For $N=1$ the inequality is an equality. See \Cref{par:statement-n-greater-1} for a discussion in greater detail.

\paragraph{Reverse KL Divergence.}
VI is known to underestimate uncertainties \citep{blei2017variational} which is due to an asymmetry in the ELBO. 
More specifically, $\text{ELBO}_\text{VI}$ (\Cref{eq:elbo}) minimizes the reverse KL divergence $\KL(q \,\|\, {f\,/\,Z})$, which takes the expectation over~$q$.
Therefore, penalties for regions in $\vz$-space where $q(\vz) > f(\vz)\,/\,Z$ are weighted higher than penalties for regions where $q(\vz) < f(\vz)\,/\,Z$.
As a result, $\text{ELBO}_\text{VI}$ is mode-seeking and tends to underestimate the true posterior variance.
For a small~$N$, IWVI also shows a mode-seeking behavior, which can be overcome by increasing $N$.
\Cref{thm:main} states that, at least for $K\to\infty$ and $N=1$, $\text{ELBO}_\text{DAIS}^{N,K}$ does not suffer from such an asymmetry.
We investigate the situation for $K<\infty$ and $N\geq1$ empirically in various experiments in \Cref{sec:experiments}.

\paragraph{Forward KL Divergence.}
In contrast, the forward KL divergence $\KL({f\,/\,Z}\,\|\,q)$ is mass-covering since the expectation is taken over~$f/Z$ and regions in $\vz$-space where $q(\vz) < f(\vz)\,/\,Z$ are weighted higher than the regions where $q(\vz) > f(\vz)\,/\,Z$.
While the forward KL divergence is less likely to underestimate posterior variances, it often performs poorly in moderate to high dimensions when posterior correlations increase \citep{dhaka2021challenges}.
\Cref{sec:experiments} provides empirical evidence that $\text{DAIS}_0$ (implicitly minimizing the symmetrized KL divergence) indeed outperforms MSC (minimizing the forward KL divergence).

\begin{figure}[!hb]
    \centering
    \includegraphics{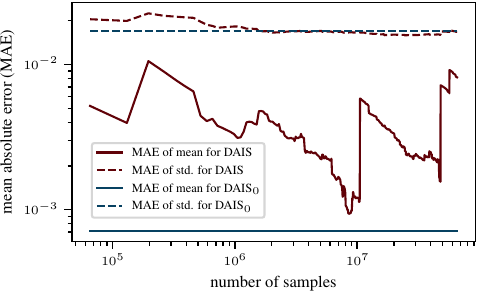}
    \caption{
    Mean absolute error of estimated mean and standard deviation as a function of the number of DAIS samples used for the estimate.
    The estimator converges poorly (see \Cref{sec:method:compactness}).
    }
    \label{fig:estimation-error-number-samples}
\end{figure}

\paragraph{Limitations.} 
We show \Cref{thm:main} for large $K$ which might be prohibitive in real world experiments.
Additionally, we generally use more than $N=1$ particles in DAIS.
Therefore, we expect that \Cref{thm:main} holds only approximately in practice.
For large~$N$, \Cref{thm:domke-sheldon} \citep[Theorem 3]{domke2018importance} holds and the gap closes.
As a result, we see a combination of effects in our experiments with practical (moderate) values for $K$ and~$N$ (see \Cref{sec:experiments}).

\subsection{Compact Representation of the Initial Distribution} \label{sec:method:compactness}

The initial distribution of DAIS, $q_0$, provides both an analytical expression (i.e., it is \textit{explicit}) and allows for tractable sampling and evaluation (i.e., it is \textit{compact}).
In the following, we discuss positive implications of this representation.

Importance sampling is likely to suffer from inefficiencies due to the sampling complexity in high dimensions.
Namely, if $f$~becomes increasingly complicated, and thus the mismatch between~${f\,/\,Z}$ and the proposal distribution~$q$ grows, the variance of the importance weight grows exponentially in its dimension \citep{bamler2017perturbative}.
\citet{agapiou2017importance} show that also the number of samples grows exponentially in the dimension of the problem.
Furthermore, \citet{chatterjee2018sample} show that the sample size is exponential in the KL divergence between the two measures (here: $\KL({f\,/\,Z}\,\|\,q)$) if they are nearly singular with respect to each other (which is often the case in practice).

As we discuss in \Cref{sec:background} above, DAIS can be seen as importance sampling on an augmented space.
Therefore, DAIS can also suffer from high sampling complexity in high dimensions.
While we should expect samples from DAIS to technically follow the target distribution more faithfully than samples from the initial distribution~$q_0$, obtaining accurate estimates of summary statistics (e.g., mean and variance) of the target distribution from DAIS samples can be prohibitively expensive in high dimensions.
By contrast, $q_0$ is typically parameterized by interpretable summary statistics that can be read off without requiring empirical estimates over exponentially many expensive MCMC samples.

In \Cref{sec:experiments} below, we observe how difficult it can be in practice to estimate summary statistics from DAIS samples.
\Cref{fig:estimation-error-number-samples}
shows the mean absolute error (MAE) of the estimated mean and standard deviation for the Gaussian process experiment ($\text{RBF}_1, d=25$ in \Cref{sec:exp:gp-inference}) as a function of the number of DAIS samples used for the estimate (red curves).
Note the sudden jumps of the estimator.
We compare to the mean and standard deviations read off directly from the initial distribution $q_0$ (horizontal blue lines).

\subsection{The Initial Distribution of DAIS for Inference}

In the following, we investigate the initial distribution~$q_0$ of DAIS as a candidate for an approximate posterior distribution.
We denote this scheme as $\text{DAIS}_0$: 
At training time, $\text{DAIS}_0$ mirrors DAIS and maximizes $\text{ELBO}_\text{DAIS}^{N,K}$.
At inference time, $\text{DAIS}_0$ drops the AIS steps and solely uses $q_0$ as a variational approximation to the target distribution.

Following the approach by \citet{zhang2021differentiable} we use DAIS with HMC dynamics.
Thus, \Cref{eq:weight-ais-simplified} simplifies to
\begin{align}
    \frac1N \sum_{i=1}^N
    \frac{f(\vz_K)}{q_0(\vz_0)}\prod_{k=1}^K \frac{\mathcal N(\mathbf{v}_k^{(i)};0,\mM)}{\mathcal N(\mathbf{v}_{k-1}^{(i)}; 0, \mM)},
\end{align}
where $\mathcal N(\mathbf{v}_k^{(i)}; 0, \mM)$ denotes the density of the normal distribution at point $\mathbf{v}_k^{(i)}$ with covariance matrix~$\mM$ and mean $0$.
$\mathbf{v}_k^{(i)}$ and $\mM$ are the momenta and the (diagonal) mass matrix of HMC, respectively.
The initial distribution~$q_0$ is a fully factorized normal distribution.
We use gradient updates to learn the means and variances of~$q_0$, the annealing schedule ($\beta_1,\ldots,\beta_{K-1}$), and the step width of the sampler.

As highlighted in \Cref{sec:method:compactness}, we do not expect that the initial distribution of DAIS ($\text{DAIS}_0$) generally outperforms the DAIS method.
However, we highlight that $\text{DAIS}_0$ provides various desirable properties due to its compact representation.
Additionally, $\text{DAIS}_0$ is much more computationally efficient at inference time compared to DAIS.
In comparison to VI, $\text{DAIS}_0$ is expected be less mode-seeking (\Cref{sec:method:theory}), and we pose that the symmetrized KL divergence is less susceptible to problems optimizing the forward 
KL divergence.
\Cref{sec:experiments} provides experimental evidence.

\section{Experiments} \label{sec:experiments}

\begin{figure*}[!t]
    \centering
    \includegraphics[width=\textwidth]{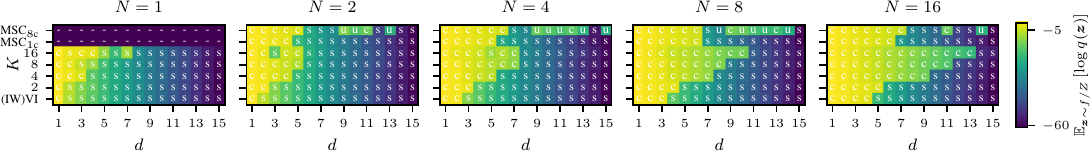}
    \vskip-0.5em
    \caption{
    Density of variational distributions of VI, IWVI, MSC, and $\text{DAIS}_0$ ($K$) evaluated on samples from a $d$-dimensional bimodal Gaussian target distribution.
    ``-'': unable to find an optimum, ``c'': mass-covering distribution, ``s'': mode-seeking distribution,
    ``u'': undecidable whether ``c'' or ``s''.
    $\text{DAIS}_0$ achieves higher densities in higher dimension~$d$ for increasing $K$ across all considered $N$.
    MSC does not converge for $N=1$.
    $\text{MSC}_{1\text{c}}$ learns variational distributions that are less mass-covering for larger $d$ than $\text{DAIS}_0\ (16)$.
    $\text{MSC}_{8\text{c}}$ achieves sometimes higher densities in higher dimension~$d$ but performs inconsistent across $N$. 
    Results are discussed in \Cref{sec:exp:d-dim-blobs}.
    }
    \label{fig:n-dim-gaussian-blobs}
\end{figure*}

In \cref{sec:exp:d-dim-blobs}, we investigate $\text{DAIS}_0$ qualitatively for various dimensions, $N$, and $K$ on toy data by analyzing its mode-seeking or mass-covering behavior.
In \cref{sec:exp:gp-inference} and \Cref{sec:exp:log-reg}, we show quantitative results on both generated and real world data.
Throughout this section, we compare $\text{DAIS}_0$, IWVI, MSC, $\text{IWVI}_{\text{SIR}}$, and DAIS, where the first three find a compact representation of the approximate posterior while the latter two require sampling at inference time.
As MSC is known to have convergence issues due to high variance \citep{kim2022markov} we report results using $8$ parallel chains as $\text{MSC}_{8\text{c}}$ and results with $1$ chain as $\text{MSC}_{1\text{c}}$. 

\begin{figure}[!b]
    \centering
    \includegraphics{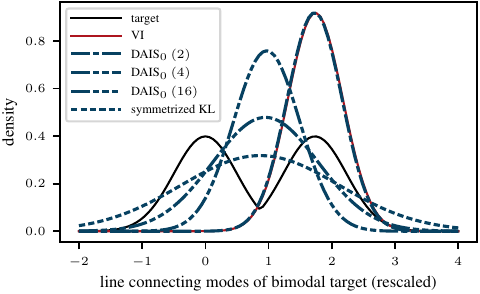}
    \caption{
    Densities on the diagonal between the two modes of the bimodal Gaussian distribution (same experiment as \Cref{fig:n-dim-gaussian-blobs} with $N=1, d=3$). 
    VI covers a single mode while $\text{DAIS}_0$ covers both modes with increasing $K$ (details in \Cref{sec:exp:d-dim-blobs}).
    }
    \label{fig:density-on-diagonal}
\end{figure}

\subsection{Bimodal Target Distribution} \label{sec:exp:d-dim-blobs}

This experiment is designed to investigate the behavior of $\text{DAIS}_0$ for $N\geq 1$ and $K\geq 1$.
More specifically, we investigate the relationship between the parameters $N$ and $K$ alongside the dimension $d$ of the variational distribution of $\text{DAIS}_0$, and compare it to the compact methods (IW(VI), MSC).
Thus, the experiment covers the entire space spanned in \Cref{fig:n-k-d-overview}.
For further details, see \Cref{sup:sec:d-dim-blob}.

We consider a bimodal target distribution $f(\vz)\,/\,Z$ that is an equally-weighted mixture of two $d$-dimensional Gaussians with means $(0,\ldots,0)$ and $(1,\ldots,1)$, respectively, and covariance matrices $0.25^2 \mI$.
We learn the variational distributions of VI, IWVI, $\text{DAIS}_0$, and MSC (both, with $1$ chain and $8$ chains).
To evaluate the learned variational distributions, we draw $1,000$ samples from the target distribution and compute the log density under~$q$, i.e., $\E_{\vz \sim f/Z}[\log(q(\vz))]$.
In \Cref{fig:n-dim-gaussian-blobs} we plot results for $N \in \{1, 2, 4, 8, 16\}$ particles and $K \in \{2, 4, 8, 16\}$ transitions in $d \in \{1, \dots, 15\}$ dimensions.
\Cref{fig:density-on-diagonal} shows the density along the line from $(0,\ldots,0)$ to $(1,\ldots,1)$ for $N=1$ and $d=3$.

\begin{figure*}[!ht]
    \centering
    \includegraphics{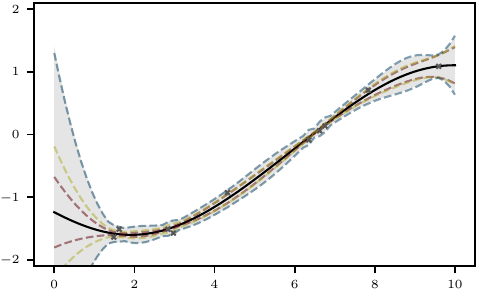}
    \hskip1em
    \includegraphics{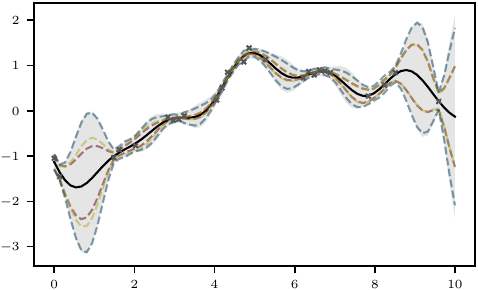}
    \vskip1ex
    \includegraphics{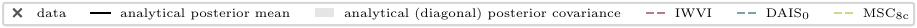}
    \caption{
    Gaussian process regression on generated data, using a prior with RBF kernel with two different sets of parameters (see $\star$~in \Cref{tab:gp-inference}). 
    We show $97.5\%$ quantiles of the posterior covariance for analytic (shaded gray; covariance matrix is diagonalized on data points, see \Cref{sup:sec:gp-inference}), IWVI (red), $\text{DAIS}_0$ (blue), and $\text{MSC}_{8\text{c}}$ (yellow).
    Learned means are indistinguishable from the analytic mean (black) at this line width.
    $\text{DAIS}_0$ often provides more accurate uncertainty estimates compared to the other methods (details in \Cref{sec:exp:gp-inference}).
    }
    \label{fig:gp-inference}
\end{figure*}

\paragraph{Mode-Seeking Versus Mass-Covering.}
In \Cref{fig:density-on-diagonal} we can clearly see that VI and $\text{DAIS}_0$ with $K=2$ transitions lead to a mode-seeking distribution.
For an increasing $K$ (i.e., $K\in\{4, 16\}$), we find that the variational distribution of $\text{DAIS}_0$ becomes more mass-covering and increasingly similar to the solution that (numerically) minimizes the symmetrized KL divergence (regularly dashed line).
This is consistent with \Cref{thm:main}.
We find that we can unambiguously classify almost all of the learned variational distributions into being mode-seeking (``s'') or mass-covering (``c'') by calculating the distance of the mean of the variational distribution to both modes of the target distribution, and to their mid point (\Cref{sup:sec:d-dim-blob} provides a histogram of these distances and more details on the classification).
\Cref{fig:n-dim-gaussian-blobs} shows corresponding classifications as labels ``s'' and ``c'' for each combination $(N, K, d)$. We label the few cases where the classification is not perfectly unambiguous with  ``u'' which is short for for ``undecidable''.

Comparing (IW)VI and $\text{DAIS}_0\ (16)$ in \Cref{fig:n-dim-gaussian-blobs}, we find that $\text{DAIS}_0\ (16)$ typically covers both modes of the target distribution for larger dimensions (across all $N$).
We attribute this finding to the symmetrized KL divergence that is implicitly minimized by $\text{DAIS}_0$.
MSC does not converge to reasonable values (``-'') for $N=1$.
With a single chain $\text{MSC}_{1\text{c}}$ outperforms (IW)VI over all $N$ but seems to be less mass-covering than $\text{DAIS}_0$.
With $8$ parallel chains $\text{MSC}_{8\text{c}}$ sometimes outperforms the other methods in finding a mass-covering distribution.
However, we find that the variational distributions of $\text{MSC}_{8\text{c}}$ are often not centered on the mid-point between the two modes which we attribute to the high variance of the method (see \Cref{sup:sec:d-dim-blob}).
This is in accordance with \citet{kim2022markov} who report high variance for MSC and  \citet{dhaka2021challenges} who note that the forward KL divergence is difficult to optimize for large dimensions.

\paragraph{Number of Transitions $K$.} 
For a fixed $N$, we find that with increasing $K$, $\text{DAIS}_0$ ($q_0$) achieves higher log densities compared to VI, IWVI, and MSC, especially in higher dimensions.
For example, for $N=4$, we find that $q_0$ covers both modes with mass for a dimension up to $d=7$ while VI and IWVI collapse to a single mode after $d=1$ and $d=2$, respectively.
These results align with our theoretical findings.
$\text{MSC}_{1\text{c}}$ collapses after $d=5$ which we attribute to the optimization of the forward KL divergence.
$\text{MSC}_{8\text{c}}$ performs better but inconsistently for higher dimensions.

\paragraph{Number of Particles $N$.} 
With an increasing number of particles~$N$, we see that IWVI and $\text{DAIS}_0$ improve.
This follows the theoretical result on the number of particles (presented in \Cref{thm:domke-sheldon} \citep[Theorem 3]{domke2018importance}) which suggests to choose $N$ as large as possible.
Also, MSC shows higher densities with an increasing $N$; for $\text{MSC}_{8\text{c}}$ the results are often inconsistent or undecidable.

\subsection{Gaussian Processes Regression} \label{sec:exp:gp-inference}

\begin{figure*}[!t]
    \centering
    \includegraphics{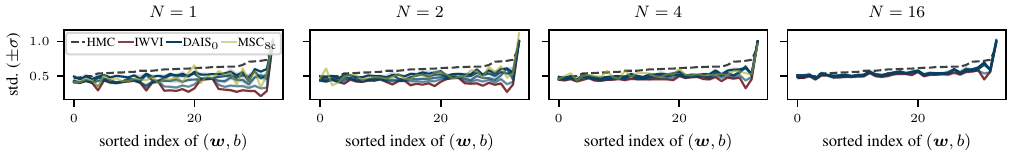}
    \vskip-0.75em
    \caption{
        Bayesian logistic regression:
        standard deviations of learned posterior weight vector (indices are plotted on the $x$-axis, we report $\pm \sigma$) for IWVI (red), $\text{MSC}_{8\text{c}}$ (yellow) and $\text{DAIS}_0$ (blue), compared to HMC samples (dashed) on \textit{ionosphere}.
        $\text{DAIS}_0$ provides uncertainty estimates that improve with increasing $K \in \{2, 4, 8, 16\}$ (shown by increasing opacity).
        Details in \Cref{sec:exp:log-reg}.
        }
    \label{fig:log-reg}
\end{figure*}

\begin{table*}[!b]
    \vskip-1.5em
    \renewcommand{\arraystretch}{1.2}
    \caption{
        Mean absolute error (compared to analytic solution) of mean and standard deviation of GP regression ($N=16$).
        $\text{DAIS}_0$~$(K)$ gives accurate uncertainty estimates for most cases (within all compact methods) while degrading the predicted mean only negligibly (details in \Cref{sec:exp:gp-inference}).
        $\text{RBF}_{\{1,2\}}$ denote different prior parameters.
        \Cref{fig:gp-inference} shows ``$\star$'' visually.
        Results are averaged over 3 runs.
    }
    \vskip 0.15in
    \small
    \centering
    \begin{tabular}{@{$\;$}cc@{$\,$}cl|ccccc@{$\;$}}
        & & $d$ & MAE &
        IWVI & $\text{IWVI}_\text{SIR}$ &
        $\text{DAIS}_0$ ($16$) & DAIS ($16$) & $\text{MSC}_{8\text{c}}$
        \\
        \shline
        \multirow{4}{*}{$\text{RBF}_1$}
         & \multirow{2}{*}{$\!\!\!\!\star$} & \multirow{2}{*}{$10$} & mean &
          $\mathit{{5.60}}_{\pm1.10}\cdot\mathit{10^{-4}}$&
          $\bm{2.83}_{\pm1.70}\cdot\bm{10^{-4}}$&
          ${1.77_{\pm.31}\cdot10^{-3}}$&
          ${7.72_{\pm.23}\cdot10^{-3}}$&
          ${1.74_{\pm.24}\cdot10^{-2}}$
         \\
         & & & std. &
          ${4.34_{\pm.0020}\cdot10^{-2}}$&
          ${3.50_{\pm.027}\cdot10^{-2}}$&
          $\mathit{4.54}_{\pm.071}\cdot\mathit{10^{-3}}$&
          $\bm{3.028}_{\pm.25}\cdot\bm{10^{-3}}$&
          ${3.66_{\pm.14}\cdot10^{-2}}$
         \\ \cline{2-9}
         & & \multirow{2}{*}{$25$} & mean &
          $\mathit{3.57}_{\pm1.10}\cdot\mathit{10^{-4}}$&
          $\bm{4.17}_{\pm4.50}\cdot\bm{10^{-5}}$&
          ${9.58_{\pm1.40}\cdot10^{-4}}$&
          ${1.39_{.036}\cdot10^{-2}}$&
          ${1.14_{\pm.16}\cdot10^{-2}}$
         \\
         & & & std. &
          ${3.83_{\pm.00030}\cdot10^{-2}}$&
          ${3.62_{\pm.042}\cdot10^{-2}}$&
          $\bm{\mathit{1.03}}_{\pm.0010}\cdot\bm{\mathit{10^{-2}}}$&
          ${2.16}_{\pm.027}\cdot{10^{-2}}$&
          ${3.77_{\pm.015}\cdot10^{-2}}$
         \\
         \hline
        \multirow{4}{*}{$\text{RBF}_2$}
         & & \multirow{2}{*}{$10$} & mean &
          $\mathit{1.77}_{\pm.14}\cdot\mathit{10^{-3}}$&
          $\bm{4.00}_{\pm1.00}\cdot\bm{10^{-5}}$&
          ${2.44_{\pm 1.10}\cdot10^{-3}}$&
          ${9.64}_{\pm1.20}\cdot{10^{-4}}$&
          ${1.57_{\pm.27}\cdot10^{-3}}$
         \\
         & & & std. &
          $\mathit{2.54}_{\pm.10}\cdot\mathit{10^{-3}}$&
          ${9.71_{\pm1.40}\cdot10^{-3}}$&
          ${6.72_{\pm.91}\cdot10^{-3}}$&
          $\bm{4.50}_{\pm.30}\cdot\bm{10^{-4}}$&
          ${7.85_{\pm2.40}\cdot10^{-4}}$
         \\ \cline{2-9}
         & \multirow{2}{*}{$\!\!\!\!\star$} & \multirow{2}{*}{$25$} & mean &
          $\mathit{5.21}_{\pm.84}\cdot\mathit{10^{-4}}$&
          $\bm{1.11}_{\pm.39}\bm{\cdot10^{-4}}$&
          ${1.15_{\pm.046}\cdot10^{-3}}$&
          ${1.99_{\pm.066}\cdot10^{-2}}$&
          ${2.16_{\pm.34}\cdot10^{-2}}$
         \\
         & & & std. &
          ${4.14_{\pm.002}\cdot10^{-2}}$&
          ${3.81_{\pm.049}\cdot10^{-2}}$&
          $\bm{\mathit{1.31}}_{\pm.0030}\cdot\bm{\mathit{10^{-2}}}$&
          $\bm{1.34}_{\pm.013}\cdot\bm{10^{-2}}$&
          ${4.07_{\pm.057}\cdot10^{-2}}$
    \end{tabular}
    \label{tab:gp-inference}
\end{table*}

We investigate a Gaussian process (GP) regression task on synthetic data.
We generate the data by (i)~sampling a ground truth process from a GP prior ${\bm f \sim \GP(\bm 0, k)}$, (ii)~randomly sampling positions~$\gI$ on the domain $[0, 10]$, and (iii)~generating data from a Gaussian likelihood model ${\bm y \sim \gN(\bm{f}_\gI, 0.1\cdot\mI)}$.
We set $k$ to be a Radial Basis Function (RBF) kernel \citep{williams2006gaussian} and investigate two sets of hyperparameters ($\text{RBF}_1$ and $\text{RBF}_2$, see \Cref{sup:sec:gp-inference}).
We learn a variational distribution for IWVI, MSC, and $\text{DAIS}_0$ on the positions~$\gI$.
While this setup might sound artificial as the analytic posterior can be computed in closed form (due to the Gaussian likelihood), one can easily think of problems with non-Gaussian likelihoods or large latent spaces that can be phrased in the same framework.

\Cref{fig:gp-inference} shows two GPs that are marked with ``$\star$'' in \Cref{tab:gp-inference}. 
We train each GP with $N=16$ particles. 
$\text{DAIS}_0$ uses $K=16$ transitions.
The black line and gray area show the theoretically best posterior mean and the corresponding posterior covariance that any Gaussian mean field approach can possibly achieve.
We calculate the posterior approximation by finding the analytic posterior on~$\gI$, diagonalizing its covariance matrix, and then calculating the predictive posterior on the full range~$[0,10]$.
By visual inspection we find that the uncertainties of $\text{DAIS}_0$ are more accurate compared to IWVI and $\text{MSC}_{8\text{c}}$.
$\text{MSC}_{8\text{c}}$ visually outperforms IWVI.

\Cref{tab:gp-inference} shows mean absolute errors (MAEs) for a larger set of GP models.
We compare models with compact distributions (IWVI, $\text{DAIS}_0$, $\text{MSC}_{8\text{c}}$) and sampling based approaches (DAIS using $10^5$ and $\text{IWVI}_\text{SIR}$ using $10^3\times10^3$ samples).
Best results are written bold.
The best method with compact representation is written in italic font.
Within the methods that have compact representations, $\text{DAIS}_0$ achieves the best MAE with respect to the standard deviation in most cases (except for $\text{RBF}_2, 10$ where relative errors are a magnitude smaller in general).
In general, $\text{DAIS}_0$ performs competitively across all methods.
Our theoretical result \Cref{thm:main} might explain this, as the symmetrized KL divergence (minimized by $\text{DAIS}_0$) does not suffer from a mode-seeking behavior but learns a $q_0$ that covers more mass of the target distribution.
For the mean estimation, IWVI outperforms $\text{DAIS}_0$ while $\text{IWVI}_\text{SIR}$ performs best.
Comparing $\text{DAIS}_0$ and DAIS, we find that a compact representation indeed helps in some cases (e.g., mean in $\text{RBF}_1, d=10,25$).
$\text{MSC}_{8\text{c}}$, minimizing the forward KL divergence, performs competitively but slightly worse in comparison to the other methods.
We attribute this finding to the instability of the forward KL divergence during optimization.
We provide a larger study in \Cref{sup:sec:gp-inference}.

\subsection{Bayesian Logistic Regression} \label{sec:exp:log-reg}

We investigate a Bayesian logistic regression task with five real-world data sets \citep{gorman1988analysis,sigillito1989classification,uci_ml_repository}.
We model the Bayesian logistic regression with a factorizing standard normal prior on weights $\vw$ and bias $b$. 
The likelihood is chosen to be Bernoulli distributed $\mathrm{Bern}(r)$ with parameter $r=\text{sigmoid}(\vw^\top \vx + b)$.
Here, $\vx$ denotes a data point.
We learn a factorizing Normal variational distribution for $\vz = (\vw, b)$.

\Cref{tab:log-reg} shows MAEs between learned means and learned standard deviations and means and standard deviations of HMC samples for various methods.
$\text{DAIS}_0$ reaches the best errors for three out of five datasets with comparable mean MAE and lower MAE for the standard deviation.
$\text{DAIS}_0$ outperforms DAIS in several cases by a significant margin (e.g., \textit{sonar} and \textit{ionosphere}) in terms of standard deviation which we mainly attribute to the larger sampling complexity of DAIS (see also \Cref{sec:method:compactness}).
Similarly, when comparing IWVI and the sample-based $\text{IWVI}_\text{SIR}$, we find that IWVI (compact) outperforms the sampling based method for larger dimensional problems.
Although $\text{MSC}_{8\text{c}}$ shows errors comparable to the other methods, it falls behind on most datasets.
The results are consistent with the previous experiments and \Cref{thm:main}.
More results can be found in \Cref{sup:sec:log-reg}.

\Cref{fig:log-reg} depicts standard deviations of HMC samples (black) and the learned standard deviations of the compact methods on the \textit{ionosphere} dataset (IWVI: red, $\text{MSC}_{8\text{c}}$: yellow, $\text{DAIS}_0$: blue, increasing opacity corresponds to increasing $K \in \{2, 4, 8, 16\}$ for training).
We use $N=16$ particles for training.
We find that $\text{DAIS}_0$ outperforms IWVI and $\text{MSC}_{8\text{c}}$ for all $N$ (see also \Cref{tab:log-reg}).
For a larger $N$, the differences between IWVI and $\text{DAIS}_0$ are visually indistinct.
$\text{MSC}_{8\text{c}}$ seems to overestimate some variances.

\begin{table*}[!t]
    \vskip-1.2em
    \renewcommand{\arraystretch}{1.2}
    \small
    \centering
    \caption{
        Mean absolute error calculated between learned mean and standard deviation to HMC samples for Bayesian logistic regression ($N=16$).
        $\text{DAIS}_0$ $(16)$ outperforms other methods on $3/5$ datasets (details in \Cref{sec:exp:log-reg}).
        Results are averaged over $3$ runs.
    }
    \vskip 0.15in
    \begin{tabular}
    {l@{\;\;\;}c@{\;\;\;}l|@{\;}c@{\;\;\;\;}c@{\;\;\;\;}c@{\;\;}c@{\;\;\;\;}c@{\;\;\;\;}c@{}}
        & $d$ & MAE &
        IWVI & $\text{IWVI}_\text{SIR}$ &
        $\text{DAIS}_0$ ($16$) & DAIS ($16$) & $\text{MSC}_{8\text{c}}$
        \\
        \shline
        \multirow{2}{*}{\!\!\!\!\textit{sonar}}
         & \multirow{2}{*}{\textit{$61$}} & mean &
          $\bm{\mathit{8.66}}_{\pm.075}\cdot\bm{\mathit{10^{-2}}}$&
          ${8.86_{\pm.17}\cdot10^{-2}}$&
          $\bm{\mathit{8.58}}_{\pm.19}\cdot\bm{\mathit{10^{-2}}}$&
          ${1.10_{\pm.089}\cdot10^{-1}}$&
          ${1.33_{\pm.081}\cdot10^{-1}}$
         \\
         && std. &
          ${7.95_{\pm.089}\cdot10^{-2}}$&
          ${8.26_{\pm.071}\cdot10^{-2}}$&
          $\bm{\mathit{4.27}}_{\pm.13}\cdot\bm{\mathit{10^{-2}}}$&
          ${6.23_{\pm.13}\cdot10^{-2}}$&
          ${7.26_{\pm.16}\cdot10^{-2}}$
         \\
         \hline
        \multirow{2}{*}{\!\!\!\!\textit{ionosphere}}
         & \multirow{2}{*}{\textit{$35$}} & mean &
          $\bm{\mathit{4.39}}_{\pm.098}\cdot\bm{\mathit{10^{-2}}}$&
          $7.18_{\pm.14}\cdot10^{-2}$&
          $\bm{\mathit{4.34}}_{\pm.023}\cdot\bm{\mathit{10^{-2}}}$&
          ${1.06_{\pm.10}\cdot10^{-1}}$&
          ${2.14_{\pm.050}\cdot10^{-1}}$
         \\
         && std. &
          ${4.73_{\pm.048}\cdot10^{-2}}$&
          ${8.34_{\pm.070}\cdot10^{-2}}$&
          $\bm{\mathit{3.25}}_{\pm.36}\cdot\bm{\mathit{10^{-2}}}$&
          ${7.64_{\pm.49}\cdot10^{-2}}$&
          ${8.55_{\pm.024}\cdot10^{-2}}$
        \\
         \hline
        \multirow{2}{*}{\!\!\!\!\textit{heart disease}}
         & \multirow{2}{*}{\textit{$16$}} & mean &
          $\bm{\mathit{2.16}}_{\pm.23}\cdot\bm{\mathit{10^{-2}}}$&
          $\bm{2.30}_{\pm.15}\cdot\bm{10^{-2}}$&
          $\bm{\mathit{2.26}}_{\pm.26}\cdot\bm{\mathit{10^{-2}}}$&
          $\bm{2.22}_{\pm.23}\cdot\bm{10^{-2}}$&
          ${3.37_{\pm.032}\cdot10^{-1}}$
         \\
         && std. &
          ${2.59_{\pm.072}\cdot10^{-2}}$&
          ${2.67_{\pm.060}\cdot10^{-2}}$&
          $\bm{\mathit{1.08}}_{\pm.28}\cdot\bm{\mathit{10^{-2}}}$&
          $\bm{1.18}_{\pm.29}\cdot\bm{10^{-2}}$&
          ${2.81_{\pm.10}\cdot10^{-2}}$
        \\
         \hline
        \multirow{2}{*}{\!\!\!\!\textit{heart attack}}
         & \multirow{2}{*}{\textit{$14$}} & mean &
          ${5.05_{\pm.75}\cdot10^{-2}}$&
          $\bm{4.81}_{\pm.57}\cdot10^{-2}$&
          $\bm{\mathit{4.68}}_{\pm.16}\cdot\bm{\mathit{10^{-2}}}$&
          $\bm{4.67}_{\pm.067}\cdot\bm{10^{-2}}$&
          ${1.75_{\pm.011}\cdot10^{-1}}$
         \\
         && std. &
          $\bm{\mathit{2.80}}_{\pm.15}\cdot\bm{\mathit{10^{-2}}}$&
          ${3.11_{\pm.14}\cdot10^{-2}}$&
          ${3.35_{\pm.11}\cdot10^{-2}}$&
          ${3.28_{\pm.094}\cdot10^{-2}}$&
          ${7.41_{\pm.35}\cdot10^{-2}}$
        \\
         \hline
        \multirow{2}{*}{\!\!\!\!\textit{loan}}
         & \multirow{2}{*}{\textit{$12$}} & mean &
          $\bm{\mathit{1.67}}_{\pm.76}\cdot\bm{\mathit{10^{-2}}}$&
          $\bm{2.38}_{\pm.47}\cdot\bm{10^{-2}}$&
          $\bm{\mathit{2.02}}_{\pm0.079}\cdot\bm{\mathit{10^{-2}}}$&
          $\bm{2.06}_{\pm.26}\cdot\bm{10^{-2}}$&
          ${1.76_{\pm.0089}\cdot10^{-1}}$
         \\
         && std. &
          $\bm{\mathit{9.23}}_{\pm.68}\cdot\bm{\mathit{10^{-3}}}$&
          ${1.14_{\pm.077}\cdot10^{-2}}$&
          ${1.28_{\pm.060}\cdot10^{-2}}$&
          ${1.32_{\pm.13}\cdot10^{-2}}$&
          ${5.08_{\pm.16}\cdot10^{-2}}$
      \end{tabular}
      \label{tab:log-reg}
\end{table*}

\section{Conclusion} \label{sec:conclusion}

In this work, we investigate the initial distribution of DAIS.
We show theoretically that, for many transition steps, it minimizes the symmetrized KL divergence to the target distribution.
Therefore, by using $q_0$ as an approximate posterior (a method that we call $\text{DAIS}_0$), we expect it to be less prone to underestimating variances (compared to VI) and easier to optimize in higher-dimensional settings (compared to MSC).
Additionally, we argue that $q_0$ is an explicit and compact representation of the target distribution, which provides advantages over sampling based methods for various downstream tasks.
In experiments on both synthetic and real-world data, we verify that $\text{DAIS}_0$ indeed often fulfills our expectations by finding distributions with more accurate variances in higher dimensions (compared to other compact and sampling based methods).
While $\text{DAIS}_0$ is more expensive than VI at training time, inference with $\text{DAIS}_0$ is just as expensive as VI but it is significantly cheaper than DAIS.

\section*{Impact Statement}
This paper presents work whose goal is to advance the field of Machine Learning. There are many potential societal consequences of our work, none which we feel must be specifically highlighted here.

\section*{Software and Data}
The datasets that we run experiments on are all either publicly available or can be generated by code.
We release the code to run all experiments at \url{https://github.com/jzenn/DAIS0}.

\section*{Acknowledgements}
Johannes Zenn would like to thank Tim Z. Xiao, Marvin Pförtner, and Tristan Cinquin for helpful discussions.
The authors would like to thank the anonymous reviewers for helpful comments and pointing out related work concerning the theoretical statement in the paper.
Funded by the Deutsche Forschungsgemeinschaft (DFG, German Research Foundation) under Germany’s Excellence Strategy – EXC number 2064/1 – Project number 390727645. 
This work was supported by the German Federal Ministry of Education and Research (BMBF): Tübingen AI Center, FKZ: 01IS18039A.
Robert Bamler acknowledges funding by the German Research Foundation (DFG) for project 448588364 of the Emmy Noether Programme.
The authors would like to acknowledge support of the `Training Center Machine Learning, Tübingen' with grant number 01|S17054.
The authors thank the International Max Planck Research School for Intelligent Systems (IMPRS-IS) for supporting Johannes Zenn.

\bibliography{example_paper}
\bibliographystyle{icml2024}

\newpage
\appendix
\onecolumn

\section{Proof of the Main Theoretical Result} \label{sup:sec:thm-main-proof}

We prove \Cref{thm:main} of the main paper.
For the reader's convenience, we restate the theorem below.

\theoremstyle{plain}
\newtheorem*{maintheorem}{\Cref*{thm:main}}

\begin{maintheorem} 
    We assume that $\operatorname{support}(q_0) \supset \operatorname{support}(f)$, perfect transitions between two annealing distributions and equally spaced $\beta_k$, i.e., $\beta_k = k/K$.
    Then, for large~$K$ and $N=1$, the gap $\Delta_\textup{AIS}^{1,K} \coloneqq \log Z - \mathrm{ELBO}_\textup{AIS}^{1,K}(f_\textup{AIS}, q_\textup{AIS})$ is a divergence between the initial distribution $q_0$ and the target distribution $f/Z$.
    Up to corrections of order $\gO(1/K^3)$, this divergence is proportional to the symmetrized KL divergence,
    \begin{align}
        \Delta_\textup{AIS}^{1,K}
        =
        \frac{1}{K}
        \left(
            \frac12 \KL\bigl(f(\vz)/Z \,\big\|\, q_0(\vz)
            \bigr)
            + 
            \frac12 \KL\bigl(q_0(\vz) \,\big\|\, f(\vz)/Z \bigr)
        \right)
        + \gO\bigl(1/K^3\bigr).
        \label{sup:eq:thm-main}
    \end{align}
\end{maintheorem}
\begin{proof}
    We start from \Cref{eq:elbo-dais} and~\Cref{eq:weight-ais-simplified} of the main text, which we reproduce here for convenience,
    \begin{align}
        \log Z &\geq
        \E_{\vz_{0:K}^{(1:N)} \sim q_\text{AIS}}
        \left[
                \log
                    \left(
                    w_\text{AIS}\bigl(\vz_{0:K}^{(1:N)}\bigr)
                \right)
        \right] 
        \eqqcolon \text{ELBO}_\text{AIS}^{N,K}(f_\text{AIS}, q_\text{AIS})
    \end{align}
    and 
    \begin{align}
        w_\text{AIS}^{N,K}\!\bigl(\vz_{0:K}^{(1:N)}\bigr)
        = \frac1N \sum_{i=1}^N \prod_{k=1}^K
        \frac{\gamma_{\beta_k}\bigl(\vz_{k-1}^{(i)}\bigr)}{\gamma_{\beta_{k-1}}\bigl(\vz_{k-1}^{(i)}\bigr)}
        \label{sup:eq:weight-ais-simplified}
    \end{align}
    \Cref{sup:eq:weight-ais-simplified} holds for DAIS if we assume perfect transitions \citep{grosse2013annealing}.
    Inserting $\gamma_{\beta_k}(\vz) = Z_{\beta_k} \pi_{\beta_k}(\vz)$, $Z_{\beta_K}=Z_1=Z$, and $Z_{\beta_0}=Z_0=1$, we obtain for $N=1$ \citep{grosse2013annealing},
    \begin{align}\label{sup:eq:Z-gap}
        \Delta_\text{AIS}^{1,K}
        = \log Z - \E_{\vz_{0:K} \sim q_\text{AIS}}
        \left[
            \log Z_{\beta_K} - \log Z_{\beta_0} +
            \sum_{k=1}^K
                \log \frac{\pi_{\beta_k}(\vz_{k-1})}{\pi_{\beta_{k-1}}(\vz_{k-1})}
        \right]
        = \sum_{k=1}^{K} \KL(\pi_{\beta_{k-1}} \,\|\, \pi_{\beta_{k}}).
    \end{align}
    As stated in the main text, $\Delta_\text{AIS}^{1,K}$ is a divergence since it is a sum of divergences and since $\Delta_\text{AIS}^{1,K}=0$ for $q_0=f/Z$ as, in this case, $\pi_{\beta_k}=q_0 \;\forall k\in\{0,\ldots,K\}$ since, by definition,
    \begin{align} \label{sup:eq:def-pi}
        \pi_{\beta_k}(\vz) \coloneqq \frac{\gamma_{\beta_k}(\vz)}{Z_{\beta_k}}
        \qquad\text{with}\qquad
        \gamma_{\beta_k}(\vz) \coloneqq q_0(\vz) \left(\frac{f(\vz)}{q_0(\vz)}\right)^{\!\beta_k}
        \qquad\text{and}\qquad
        Z_{\beta_k} = \int\! \gamma_{\beta_k}(\vz)\,\mathrm{d}\vz.
    \end{align}
    
    We now show \Cref{sup:eq:thm-main}.
    For each term on the right-hand side of \Cref{sup:eq:Z-gap}, we have 
    \begin{align}\label{sup:eq:def-h}
        \KL(\pi_{\beta_{k-1}} \,\|\, \pi_{\beta_{k}}) = h_{k-1}(\beta_k) - h_{k-1}(\beta_{k-1}), 
        \qquad\text{where}\qquad 
        h_{k-1}(\eta) \coloneqq -\E_{\pi_{\beta_{k-1}}(\vz)}[\log \pi_{\eta}(\vz)].
    \end{align}

    Using Taylor's theorem, we can express
    \begin{align}
        h_{k-1}(\beta_k) - h_{k-1}(\beta_{k-1})
        =
        (\beta_k - \beta_{k-1}) h'_{k-1}(\beta_k) + R_1(\beta_k),
        \label{sup:eq:taylor-expansion}
    \end{align}
    where primes denote derivatives.
    We use the Lagrange from for the remainder
    \begin{align}
        R_1(\beta_k) 
        =
        \frac{1}{2}
        (\beta_k - \beta_{k-1})^2
        h''_{k-1}(\eta_{k-1}),
    \end{align}
    for some $\eta_{k-1}\in [\beta_{k-1},\beta_k]$.
    
    We can show that the first term of the right-hand side of \Cref{sup:eq:taylor-expansion} is zero by
    \begin{align}\label{sup:eq:first-derivative-zero}
        \begin{split}
        h_{k-1}'(\beta_{k-1})
        &=
        -\nabla_\eta
        \left(\E_{\pi_{\beta_{k-1}}(\vz)}[\log \pi_{\eta}(\vz)]\right)_{\!\eta=\beta_{k-1}}
        \\
        &=
        -\int\! \pi_{\beta_{k-1}}(\vz) \frac{\nabla_\eta \left(\pi_{\eta}(\vz)\right)_{\eta=\beta_{k-1}}}{\pi_{\beta_{k-1}}(\vz)}\,\mathrm{d}\vz
        =
        -\nabla_\eta\left( \int\! \pi_{\eta}(\vz) \,\mathrm{d}\vz\right)_{\!\!\!\;\eta=\beta_{k-1}}
        =
        -\nabla_\eta(1) = 0.
        \end{split}
    \end{align}
    Therefore, only the second term on the right-hand side of \Cref{sup:eq:taylor-expansion} contributes.

    From \Cref{sup:eq:def-pi}, we combine the leftmost and the middle equation and arrive at
    \begin{align}\label{sup:eq:combine-unnormalized-normalized}
        \log \pi_\eta(\vz) = \log q_0(\vz) - \eta \bigl(\log f(\vz) - \log q_0(\vz)\bigr) - \log Z_\eta.
    \end{align}
    Inserting \Cref{sup:eq:combine-unnormalized-normalized} into \Cref{sup:eq:def-h} and taking the second derivative we get
    \begin{align}\label{sup:eq:second-derivative}
        h_{k-1}''(\eta_{k-1}) = \partial^2 \log Z_{\eta_{k-1}} \,/\, \partial \eta_{k-1}^2.
    \end{align}

    Inserting \Cref{sup:eq:first-derivative-zero} and \Cref{sup:eq:second-derivative} into \Cref{sup:eq:taylor-expansion} and the result into \Cref{sup:eq:Z-gap}, we find 
    \begin{align}
        \Delta_\text{AIS}^{1,K} &= 
        \sum_{k=1}^{K} \KL(\pi_{\beta_{k-1}} \,\|\, \pi_{\beta_{k}})
        =
        \frac{1}{2K}\Biggl(\frac1K \sum_{k=1}^{K}\frac{\partial^2 \log Z_{\eta_{k-1}}}{\partial \eta_{k-1}^2} \Biggr).
    \end{align}
    Here, the term in the parentheses approximates an integral (since $\eta_{k-1}\in[\beta_{k-1},\beta_k] = \bigl[\frac{k-1}{K},\frac{k}{K}\bigr]\;\forall k$).
    Using that the error for such numerical integration scales as $\gO(1/K^2)$, we thus find
    \begin{align}
        \Delta_\text{AIS}^{1,K}
        &=\frac{1}{2K} 
        \left(
        \int_0^1
        \frac{\partial^2\log Z_{\eta}}{\partial^2 \eta}
        \,\mathrm{d}\eta + \gO\bigl(1/K^2\bigr)
        \right)
        \label{sup:eq:reverse-taylor}
        \\
        &=
            \frac{1}{2K} \left(
                \left(\frac{\partial \log Z_\eta}{\partial\eta}\right)_{\!\eta=1}
                -\left(\frac{\partial \log Z_\eta}{\partial\eta}\right)_{\!\eta=0}
            \right)
            + \gO\bigl(1/K^3\bigr).
        \label{sup:eq:gap-from-boundary-terms}
    \end{align}
    
    Thus, for large~$K$, only the boundary terms remain.
    We calculate them by inserting the definitions of~$Z_\eta$ from \Cref{eq:def-pi} and using the relation ${\partial x^\eta/\partial\eta} = x^\eta \log x$.
    We find
    \begin{align}
        \frac{\partial \log Z_\eta}{\partial\eta}
        &=
        \frac{1}{Z_\eta}
        \int
        q_0(\vz) 
        \left(
            \frac{f(\vz)}{q_0(\vz)}
        \right)^{\!\eta}
        \log \frac{f(\vz)}{q_0(\vz)}
        \,\mathrm{d}\vz.
         \label{sup:eq:eta-0-without-log}
    \end{align}
    Thus, recalling that $Z_0=1$ and $Z_1=Z$ are the normalization constants of $q_0$ and~$f$, respectively,
    \begin{align}
        \left(\frac{\partial \log Z_\eta}{\partial\eta}\right)_{\!\eta=1} &= \E_{\vz\sim f(\vz)/Z}\left[\log\frac{f(\vz)}{q_0(\vz)}\right] = \KL\bigl(f(\vz)/Z \,\big\|\, q_0(\vz)\bigr); \label{sup:eq:kl-at-eta1}
        \\
        -\left(\frac{\partial \log Z_\eta}{\partial\eta}\right)_{\!\eta=0} &= \E_{\vz\sim q_0(\vz)}\left[\log\frac{q_0(\vz)}{f(\vz)}\right] = \KL\bigl(q_0(\vz) \,\big\|\, f(\vz)/Z \bigr).\label{sup:eq:kl-at-eta0}
    \end{align}
    Inserting \Cref{sup:eq:kl-at-eta1,sup:eq:kl-at-eta0} into \Cref{sup:eq:gap-from-boundary-terms} leads to the claim in \Cref{sup:eq:thm-main}.
\end{proof}

\paragraph{Statement for $N>1$.} \label{par:statement-n-greater-1}
Starting from \Cref{sup:eq:weight-ais-simplified} we obtain for the gap with general $N$,
\begin{align}
    \begin{split}
        \Delta_\mathrm{AIS}^{N,K} = 
        -\E_{\vz_{0:K}^{(1:N)} \sim q_\text{AIS}}
        \left[\log \left( 
        \frac{1}{N} 
        \sum^N_{i=1} 
        \prod_{k=1}^K 
        \frac{\gamma_{\beta_k}\bigl(\vz_{k-1}^{(i)}\bigr)}
        {\gamma_{\beta_{k-1}}\bigl(\vz_{k-1}^{(i)}\bigr)} 
        \right)\right].
    \end{split}
\end{align}
Due to the sum over $N$, the right-hand side can now no longer be separated into a sum of $K$ terms. 
However, we can still bound the right-hand side by pulling the sum out of the logarithm using Jensen’s inequality, resulting in 
\begin{align}
    \Delta_\text{AIS}^{N,K} 
    \leq 
    \Delta_\text{AIS}^{1,K},
\end{align}
which implies
\begin{align}
    \lim_{K\to\infty} \Delta_\text{AIS}^{N,K} 
    \leq 
    \frac 1K
    D_\text{JS}(q_0, f/Z)
    + \gO\bigl(1/K^3\bigr),
\end{align}
where $D_\text{JS}$ denotes the symmetrized KL divergence.
This shows that the bound for $N>1$ can be upper-bounded by the symmetrized KL divergence but does not give additional insights on, e.g., symmetry. 
For $N=1$ the inequality is an equality.

\section{Multidimensional Bimodal Target Distributions} \label{sup:sec:d-dim-blob} 

We provide further details for the experiment on multidimensional bimodal Gaussian target distributions.
The experiment is discussed in \Cref{sec:exp:d-dim-blobs} of the main text.
We use an Intel XEON CPU E5-2650 v4 for running the small-scale experiments and a single NVIDIA GeForce GTX 1080 Ti for the large-scale experiments.

The model $p$ is defined as follows
\begin{align}
    p(\vz) = 
    \frac{1}{2}\,\gN(\vz; \bm 0, 0.25^2\mI) + 
    \frac{1}{2}\,\gN(\vz; \bm 1, 0.25^2\mI).
\end{align}
We initialize the variational distributions of VI, IWVI and $\text{DAIS}_0$ with 
\begin{align}
    q_\star(\vz) = \gN\left(\vz; \frac12 \bm{1}, \mI\right),
\end{align}
where the mean and the diagonal of the covariance matrix are learnable parameters.
We train the models for $7,500$ iterations with the Adam optimizer \citep{kingma2014adam} and a learning rate of $10^{-2}$.

DAIS utilizes the parameterization of \citet{zhang2021differentiable}.
In addition to the parameters of $q_0$, we learn the scale $c$ of the mass matrix $\mM = c\mI$, the annealing schedule ($\beta_1, \dots, \beta_{K-1}$), and the step widths of the sampler.

MSC is trained for a comparable number of iterations with a learning rate of $10^{-4}$.

\paragraph{Classification Into Mode-Seeking ``s'' and Mass-Covering ``c'' and Undecidable ``u''.}
To get a first classification on ``c'' or ``s'', we measure the distance of the mean of the variational distribution to both modes of the target distribution, and to their mid point $(0.5,\ldots,0.5)$. 
It turns out that this criterion clearly identifies in almost all cases whether a distribution is mode-seeking (see \Cref{sup:fig:histogram-differences}). 
For cases where we do not clearly identify whether a distribution is mass-covering or mode-seeking, we use the letter ``u'', short for ``undecidable''.
Afterwards, we manually verified the classifications by making a plot similar to \Cref{fig:density-on-diagonal} for each ``square'' of \Cref{fig:n-dim-gaussian-blobs}.

\begin{figure*}[h]
    \centering
    \includegraphics[width=\textwidth]{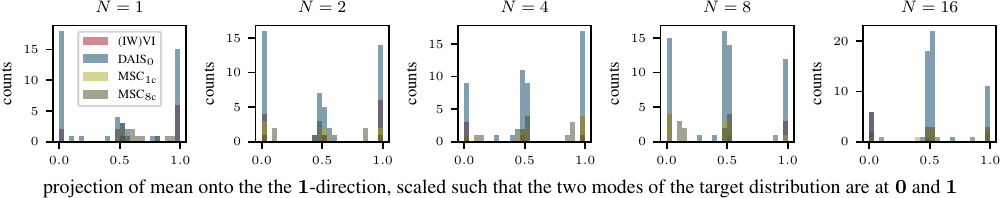}
    \caption{
        Distance of the mean of the variational distribution to both modes of $p$ and its mid point.
        Distances are mostly $0.0$, $1.0$, or $0.5$ for (IW)VI, $\text{DAIS}_0$, and $\text{MSC}_{1\text{c}}$  indicating that the mean of the variational distribution falls on one of the two modes or the mid point.
        For $\text{MSC}_{8\text{c}}$ there are more outliers.
    }
    \label{sup:fig:histogram-differences}
\end{figure*}

\clearpage
\section{Gaussian Process Regression} \label{sup:sec:gp-inference}

We provide further details on model, the joint distribution we plot in the main text, and quantitative results of the GP regression experiment on synthetic data discussed in \Cref{sec:exp:gp-inference} of the main text.
We use an Intel XEON CPU E5-2650 v4 for running the small-scale experiments and a single NVIDIA GeForce GTX 1080 Ti for the large-scale experiments.

\subsection{Model}

The kernels, called $\text{RBF}_1$ and $\text{RBF}_2$ in the main text, are instances of the following RBF kernel with different lengthscales $\rho_1 = 0.8$ and $\rho_2 = 3.0$ and 
\begin{align}
    k(t, s) =
    \exp\left(
        - \frac{(t - s)^2}{2\rho^2}
    \right).
\end{align}
We discretize the domain $[0, 10]$ on $75$ points.
We initialize the variational distributions of VI, IWVI and $\text{DAIS}_0$ with standard normal distributions.

All models are trained for $50,000$ iterations with the Adam optimizer and a learning rate of $10^{-3}$.
DAIS utilizes the parameterization of \citet{zhang2021differentiable}.
In addition to the parameters of $q_0$, we learn the diagonal $\vd$ of the mass matrix $\mM = \mathrm{diag}(\vd)$, the annealing schedule ($\beta_1, \dots, \beta_{K-1}$), and the step widths of the sampler.

\subsection{Joint Distribution}

This section provides more background on how we produce the plots of the GP regression experiment in the main text.

Let $\vf$ denote the latent process of interest.
Let $\vf_o$ denote the part of $\vf$ that we have data on and let $\vf_u$ denote the remaining part.
The prior on $\vf$ can then be written as follows.
\begin{align}
    p(\vf) &= 
    p\left(
        \begin{pmatrix}
            \vf_o \\ \vf_u
        \end{pmatrix}
    \right)
    = 
    \gN\left(
        \begin{pmatrix}
            \vf_o \\ \vf_u
        \end{pmatrix};
        \begin{pmatrix}
            \vm_o \\ \vm_u
        \end{pmatrix},
        \begin{pmatrix}
            \vS_{o,o} & \vS_{o,u} \\ 
            \vS_{u,o} &  \vS_{u,u}
        \end{pmatrix}
    \right) \\
    p(\vf_o) p(\vf_u \mid \vf_o) 
    &= \gN(\vf_o; \vm_o, \vS_{o,o})
    \gN(\vf_u; \vm_u + \vS_{u,o} \vS_{o,o}^{-1}(\vf_o - \vm_o), \vS_{u,u} - \vS_{u,o}\vS_{o,o}^{-1}\vS_{o,u})
\end{align}
We model the data $\vy$ with a Gaussian likelihood model with fixed variance $\sigma^2 = 0.1$.
\begin{align}
    p(\vf_o \mid \vy) = \gN(\vy; {f_o}, \sigma^2\mI)
\end{align}
If we condition on data $\vy$ we compute the joint as follows.
\begin{align}
    p(\vf \mid \vy) &=
    p(\vf_o \mid \vy) p(\vf_u \mid \vf_o, \vy) =
    p(\vf_o \mid \vy) p(\vf_u \mid \vf_o) \\
    \begin{split}
    &= \gN\left(
        \begin{pmatrix}
            \vf_o \\ \vf_u
        \end{pmatrix};
        \begin{pmatrix}
            \vm^+ \\ \vS_{u,o}\vS_{o,o}^{-1}\vm_u^+ + \vS_{u,o}\vS_{o,o}^{-1}\vm_o + \vm_u
        \end{pmatrix},
        \right.\\
        &\qquad\qquad\qquad
        \left.
        \begin{pmatrix}
            \vS^+ & \vS^+(\vS_{u,o}\vS_{o,o}^{-1})^\top \\ 
            \vS_{u,o}\vS_{o,o}^{-1}\vS^+ & 
            (\vS_{u,o}\vS_{o,o}^{-1}\vS^+(\vS_{u,o}\vS_{o,o}^{-1})^\top + \vS_{u,u} - \vS_{u,o}\vS_{o,o}^{-1}\vS_{o,u}
        \end{pmatrix}
    \right)
    \end{split} \label{sup:eq:joint}
\end{align}
where $\vm^+$ and $\vS^+$ are either (a) the analytically inferred mean and variance of the posterior Gaussian given the data or (b) the learned variational approximation given data.

For the first case, we get
\begin{align}
    \vm^+ &= \vm_o + \vS_{o,o} (\vS_{o,o} + \sigma^2 \mI)^{-1} (\vx - \vm_o) \quad\text{and} \\
    \vS^+ &= \vS_{o,o} - \vS_{o,o}(\vS_{o,o} + \sigma^2 \mI)^{-1}\vS_{o,o}.
\end{align}
In the second case, we model the Gaussian distribution by a mean vector and a variance vector
\begin{align}
    \vm^+ = \check{\vm} \quad\text{and}\quad
    \vS^+ = \mathrm{diag}(\check{\vs}).
\end{align}

We plot the shaded region by calculating \Cref{sup:eq:joint} with
\begin{align}
    \vS^+ = \mathrm{diag}(\mathrm{diag}(\vS_{o,o} - \vS_{o,o}(\vS_{o,o} + \sigma^2 \mI)^{-1}\vS_{o,o})).
\end{align}

\subsection{Complementary Quantitative Results}

\Cref{sup:tab:gp-inference-0}, \Cref{sup:tab:gp-inference-0A}, \Cref{sup:tab:gp-inference-1}, \Cref{sup:tab:gp-inference-1A}, \Cref{sup:tab:gp-inference-2}, \Cref{sup:tab:gp-inference-2A}, \Cref{sup:tab:gp-inference-3}, \Cref{sup:tab:gp-inference-3A}, \Cref{sup:tab:gp-inference-4},  \Cref{sup:tab:gp-inference-4-A} provide additional results on our the experiment.

\begin{table}[!ht]
    \renewcommand{\arraystretch}{1.4}
    \tiny
    \centering
    \caption{
        MAE (compared to analytic solution) of mean and standard deviation of GP regression.
        More details are provided in \Cref{sec:exp:gp-inference}.
        Results are averaged over 3 runs.
        $\text{RBF}_{\{1,2\}}$: different prior parameters.
    }
\begin{tabular}{ccl|ccccc}
        \multicolumn{3}{c}{ }
        & \multicolumn{5}{c}{$N=1$}
        \\
        & $d$ & MAE &
        IWVI & $\text{DAIS}_0$ ($2$) & $\text{DAIS}_0$ ($4$) & 
        $\text{DAIS}_0$ ($8$) & $\text{DAIS}_0$ ($16$)
        \\
        \shline
        \multirow{4}{*}{$\text{RBF}_1$}
         & \multirow{2}{*}{$10$} & mean &
          ${2.29\cdot10^{-4}}_{\pm3.9\cdot10^{-5}}$&
          ${2.30\cdot10^{-4}}_{\pm7.9\cdot10^{-5}}$&
          ${6.65\cdot10^{-4}}_{\pm9.5\cdot10^{-5}}$&
          ${1.30\cdot10^{-3}}_{\pm6.9\cdot10^{-5}}$&
          ${6.11\cdot10^{-4}}_{\pm2.2\cdot10^{-4}}$
         \\
         & & std. &
          ${5.16\cdot10^{-2}}_{\pm1.7\cdot10^{-5}}$&
          ${4.84\cdot10^{-2}}_{\pm1.0\cdot10^{-5}}$&
          ${3.97\cdot10^{-2}}_{\pm2.1\cdot10^{-5}}$&
          ${3.18\cdot10^{-2}}_{\pm3.0\cdot10^{-5}}$&
          ${2.51\cdot10^{-2}}_{\pm1.3\cdot10^{-5}}$
         \\ \cline{2-8}
         & \multirow{2}{*}{$25$} & mean &
          ${1.73\cdot10^{-4}}_{\pm1.3\cdot10^{-5}}$&
          ${1.58\cdot10^{-4}}_{\pm2.2\cdot10^{-5}}$&
          ${2.46\cdot10^{-4}}_{\pm4.8\cdot10^{-5}}$&
          ${4.95\cdot10^{-4}}_{\pm9.0\cdot10^{-5}}$&
          ${5.53\cdot10^{-4}}_{\pm6.5\cdot10^{-5}}$
         \\
         & & std. &
          ${4.03\cdot10^{-2}}_{\pm2.8\cdot10^{-6}}$&
          ${3.97\cdot10^{-2}}_{\pm5.1\cdot10^{-6}}$&
          ${3.79\cdot10^{-2}}_{\pm4.0\cdot10^{-6}}$&
          ${3.51\cdot10^{-2}}_{\pm1.4\cdot10^{-4}}$&
          ${2.61\cdot10^{-2}}_{\pm1.2\cdot10^{-5}}$
         \\
         \hline
        \multirow{4}{*}{$\text{RBF}_2$}
         & \multirow{2}{*}{$10$} & mean &
          ${1.08\cdot10^{-3}}_{\pm2.6\cdot10^{-4}}$&
          ${9.84\cdot10^{-4}}_{\pm1.5\cdot10^{-4}}$&
          ${1.39\cdot10^{-3}}_{\pm9.5\cdot10^{-5}}$&
          ${1.26\cdot10^{-3}}_{\pm2.9\cdot10^{-4}}$&
          ${1.57\cdot10^{-3}}_{\pm5.8\cdot10^{-4}}$
         \\
\end{tabular}
    \label{sup:tab:gp-inference-0}
\end{table}

\begin{table}[!ht]
    \renewcommand{\arraystretch}{1.4}
    \tiny
    \centering
    \caption{
        MAE (compared to analytic solution) of mean and standard deviation of GP regression.
        More details are provided in \Cref{sec:exp:gp-inference}.
        Results are averaged over 3 runs.
        $\text{RBF}_{\{1,2\}}$: different prior parameters.
    }
    \begin{tabular}{ccl|ccccc}
        \multicolumn{3}{c}{ }
        & \multicolumn{5}{c}{$N=1$}
        \\
        & $d$ & MAE &
        $\text{IWVI}_\text{SIR}$ & DAIS ($2$) & DAIS ($4$) & DAIS ($8$) & DAIS ($16$)
        \\
        \shline
        \multirow{4}{*}{$\text{RBF}_1$}
         & \multirow{2}{*}{$10$} & mean &
          ${1.13\cdot10^{-3}}_{\pm5.0\cdot10^{-5}}$&${4.64\cdot10^{-3}}_{\pm1.9\cdot10^{-4}}$&${2.19\cdot10^{-2}}_{\pm3.1\cdot10^{-4}}$&${7.55\cdot10^{-3}}_{\pm4.5\cdot10^{-4}}$&${1.16\cdot10^{-2}}_{\pm1.8\cdot10^{-4}}$
         \\
         & & std. &
          ${1.48\cdot10^{-2}}_{\pm1.1\cdot10^{-4}}$&${4.72\cdot10^{-2}}_{\pm9.2\cdot10^{-5}}$&${2.40\cdot10^{-2}}_{\pm4.1\cdot10^{-5}}$&${6.95\cdot10^{-3}}_{\pm2.3\cdot10^{-4}}$&${2.25\cdot10^{-2}}_{\pm1.4\cdot10^{-4}}$
         \\ \cline{2-8}
         & \multirow{2}{*}{$25$} & mean &
          ${2.44\cdot10^{-4}}_{\pm1.4\cdot10^{-5}}$&${3.20\cdot10^{-3}}_{\pm7.3\cdot10^{-5}}$&${2.29\cdot10^{-3}}_{\pm2.9\cdot10^{-4}}$&${5.78\cdot10^{-3}}_{\pm2.7\cdot10^{-3}}$&${4.41\cdot10^{-3}}_{\pm3.8\cdot10^{-4}}$
         \\
         & & std. &
          ${2.75\cdot10^{-2}}_{\pm2.4\cdot10^{-4}}$&${3.97\cdot10^{-2}}_{\pm5.1\cdot10^{-5}}$&${3.16\cdot10^{-2}}_{\pm9.0\cdot10^{-5}}$&${2.17\cdot10^{-2}}_{\pm1.5\cdot10^{-3}}$&${2.71\cdot10^{-2}}_{\pm3.6\cdot10^{-5}}$
         \\
         \hline
        \multirow{4}{*}{$\text{RBF}_2$}
         & \multirow{2}{*}{$10$} & mean &
          ${1.54\cdot10^{-5}}_{\pm5.2\cdot10^{-6}}$&${1.15\cdot10^{-2}}_{\pm4.7\cdot10^{-4}}$&${1.90\cdot10^{-3}}_{\pm3.0\cdot10^{-5}}$&${1.87\cdot10^{-3}}_{\pm4.0\cdot10^{-5}}$&${2.78\cdot10^{-3}}_{\pm9.3\cdot10^{-5}}$
         \\
         & & std. &
          ${1.69\cdot10^{-3}}_{\pm9.4\cdot10^{-5}}$&${6.40\cdot10^{-3}}_{\pm1.1\cdot10^{-4}}$&${2.25\cdot10^{-3}}_{\pm1.6\cdot10^{-5}}$&${1.64\cdot10^{-3}}_{\pm1.4\cdot10^{-5}}$&${1.99\cdot10^{-3}}_{\pm2.0\cdot10^{-4}}$
         \\
\end{tabular}
    \label{sup:tab:gp-inference-0A}
\end{table}

\begin{table}[!ht]
    \renewcommand{\arraystretch}{1.4}
    \tiny
    \centering
    \caption{
        MAE (compared to analytic solution) of mean and standard deviation of GP regression.
        More details are provided in \Cref{sec:exp:gp-inference}.
        Results are averaged over 3 runs.
        $\text{RBF}_{\{1,2\}}$: different prior parameters.
    }
    \begin{tabular}{ccl|ccccc}
        \multicolumn{3}{c}{ }
        & \multicolumn{5}{c}{$N=2$}
        \\
        & $d$ & MAE &
        IWVI & $\text{DAIS}_0$ ($2$) & $\text{DAIS}_0$ ($4$) & 
        $\text{DAIS}_0$ ($8$) & $\text{DAIS}_0$ ($16$)
        \\
        \shline
        \multirow{4}{*}{$\text{RBF}_1$}
         & \multirow{2}{*}{$10$} & mean &
          ${4.21\cdot10^{-4}}_{\pm9.6\cdot10^{-5}}$&
          ${3.19\cdot10^{-4}}_{\pm6.8\cdot10^{-5}}$&
          ${7.78\cdot10^{-4}}_{\pm3.4\cdot10^{-4}}$&
          ${6.10\cdot10^{-4}}_{\pm3.6\cdot10^{-4}}$&
          ${8.30\cdot10^{-4}}_{\pm2.2\cdot10^{-4}}$
         \\
         & & std. &
          ${4.95\cdot10^{-2}}_{\pm2.1\cdot10^{-5}}$&
          ${4.53\cdot10^{-2}}_{\pm5.8\cdot10^{-6}}$&
          ${3.59\cdot10^{-2}}_{\pm1.0\cdot10^{-5}}$&
          ${2.76\cdot10^{-2}}_{\pm2.5\cdot10^{-5}}$&
          ${1.84\cdot10^{-2}}_{\pm6.8\cdot10^{-5}}$
         \\ \cline{2-8}
         & \multirow{2}{*}{$25$} & mean &
          ${1.77\cdot10^{-4}}_{\pm1.7\cdot10^{-5}}$&
          ${2.53\cdot10^{-4}}_{\pm2.9\cdot10^{-5}}$&
          ${4.10\cdot10^{-4}}_{\pm5.8\cdot10^{-5}}$&
          ${3.60\cdot10^{-4}}_{\pm9.6\cdot10^{-5}}$&
          ${6.93\cdot10^{-4}}_{\pm5.0\cdot10^{-5}}$
         \\
         & & std. &
          ${3.98\cdot10^{-2}}_{\pm4.9\cdot10^{-6}}$&
          ${3.89\cdot10^{-2}}_{\pm4.1\cdot10^{-6}}$&
          ${3.75\cdot10^{-2}}_{\pm3.4\cdot10^{-5}}$&
          ${3.39\cdot10^{-2}}_{\pm7.5\cdot10^{-5}}$&
          ${2.12\cdot10^{-2}}_{\pm4.7\cdot10^{-5}}$
         \\
         \hline
        \multirow{4}{*}{$\text{RBF}_2$}
         & \multirow{2}{*}{$10$} & mean &
          ${1.36\cdot10^{-3}}_{\pm2.5\cdot10^{-4}}$&
          ${1.02\cdot10^{-3}}_{\pm1.5\cdot10^{-4}}$&
          ${2.01\cdot10^{-3}}_{\pm6.8\cdot10^{-4}}$&
          ${1.45\cdot10^{-3}}_{\pm1.2\cdot10^{-4}}$&
          ${1.47\cdot10^{-3}}_{\pm5.3\cdot10^{-5}}$
         \\
\end{tabular}
    \label{sup:tab:gp-inference-1}
\end{table}

\begin{table}[!ht]
    \renewcommand{\arraystretch}{1.4}
    \tiny
    \centering
    \caption{
        MAE (compared to analytic solution) of mean and standard deviation of GP regression.
        More details are provided in \Cref{sec:exp:gp-inference}.
        Results are averaged over 3 runs.
        $\text{RBF}_{\{1,2\}}$: different prior parameters.
    }
    \begin{tabular}{ccl|ccccc}
        \multicolumn{3}{c}{ }
        & \multicolumn{5}{c}{$N=2$}
        \\
        & $d$ & MAE &
        $\text{IWVI}_\text{SIR}$ & DAIS ($2$) & DAIS ($4$) & DAIS ($8$) & DAIS ($16$)
        \\
        \shline
        \multirow{4}{*}{$\text{RBF}_1$}
         & \multirow{2}{*}{$10$} & mean &
          ${8.62\cdot10^{-4}}_{\pm4.5\cdot10^{-5}}$&${5.69\cdot10^{-3}}_{\pm3.3\cdot10^{-4}}$&${1.67\cdot10^{-2}}_{\pm1.3\cdot10^{-3}}$&${1.15\cdot10^{-2}}_{\pm3.2\cdot10^{-4}}$&${1.64\cdot10^{-2}}_{\pm8.9\cdot10^{-4}}$
         \\
         & & std. &
          ${1.12\cdot10^{-2}}_{\pm4.6\cdot10^{-4}}$&${4.44\cdot10^{-2}}_{\pm5.3\cdot10^{-5}}$&${1.96\cdot10^{-2}}_{\pm3.5\cdot10^{-4}}$&${2.58\cdot10^{-2}}_{\pm5.5\cdot10^{-5}}$&${2.12\cdot10^{-2}}_{\pm2.0\cdot10^{-4}}$
         \\ \cline{2-8}
         & \multirow{2}{*}{$25$} & mean &
          ${7.90\cdot10^{-4}}_{\pm4.5\cdot10^{-6}}$&${4.04\cdot10^{-3}}_{\pm7.8\cdot10^{-4}}$&${5.74\cdot10^{-3}}_{\pm4.9\cdot10^{-4}}$&${8.46\cdot10^{-3}}_{\pm1.0\cdot10^{-4}}$&${3.44\cdot10^{-3}}_{\pm4.7\cdot10^{-5}}$
         \\
         & & std. &
          ${3.54\cdot10^{-2}}_{\pm4.6\cdot10^{-4}}$&${3.94\cdot10^{-2}}_{\pm6.6\cdot10^{-4}}$&${3.27\cdot10^{-2}}_{\pm2.4\cdot10^{-4}}$&${3.64\cdot10^{-2}}_{\pm1.2\cdot10^{-5}}$&${2.69\cdot10^{-2}}_{\pm1.7\cdot10^{-4}}$
         \\
         \hline
        \multirow{4}{*}{$\text{RBF}_2$}
         & \multirow{2}{*}{$10$} & mean &
          ${1.00\cdot10^{-5}}_{\pm2.0\cdot10^{-6}}$&${7.30\cdot10^{-3}}_{\pm1.9\cdot10^{-4}}$&${1.97\cdot10^{-3}}_{\pm7.6\cdot10^{-6}}$&${1.85\cdot10^{-3}}_{\pm5.7\cdot10^{-5}}$&${2.52\cdot10^{-3}}_{\pm3.7\cdot10^{-5}}$
         \\
         & & std. &
          ${1.64\cdot10^{-3}}_{\pm1.3\cdot10^{-4}}$&${3.41\cdot10^{-3}}_{\pm1.5\cdot10^{-4}}$&${2.23\cdot10^{-3}}_{\pm3.0\cdot10^{-5}}$&${1.71\cdot10^{-3}}_{\pm1.2\cdot10^{-5}}$&${1.95\cdot10^{-3}}_{\pm7.7\cdot10^{-5}}$
         \\
\end{tabular}
    \label{sup:tab:gp-inference-1A}
\end{table}

\begin{table}[!ht]
    \renewcommand{\arraystretch}{1.4}
    \tiny
    \centering
    \caption{
        MAE (compared to analytic solution) of mean and standard deviation of GP regression.
        More details are provided in \Cref{sec:exp:gp-inference}.
        Results are averaged over 3 runs.
        $\text{RBF}_{\{1,2\}}$: different prior parameters.
    }
    \begin{tabular}{ccl|ccccc}
        \multicolumn{3}{c}{ }
        & \multicolumn{5}{c}{$N=4$}
        \\
        & $d$ & MAE &
        IWVI & $\text{DAIS}_0$ ($2$) & $\text{DAIS}_0$ ($4$) & $\text{DAIS}_0$ ($8$) & $\text{DAIS}_0$ ($16$)
        \\
        \shline
        \multirow{4}{*}{$\text{RBF}_1$}
         & \multirow{2}{*}{$10$} & mean &
          ${3.09\cdot10^{-4}}_{\pm9.1\cdot10^{-5}}$&
          ${4.36\cdot10^{-4}}_{\pm6.9\cdot10^{-5}}$&
          ${7.07\cdot10^{-4}}_{\pm1.6\cdot10^{-4}}$&
          ${7.60\cdot10^{-4}}_{\pm2.2\cdot10^{-4}}$&
          ${8.57\cdot10^{-4}}_{\pm1.8\cdot10^{-4}}$
         \\
         & & std. &
          ${4.75\cdot10^{-2}}_{\pm2.5\cdot10^{-5}}$&
          ${4.24\cdot10^{-2}}_{\pm2.8\cdot10^{-5}}$&
          ${3.21\cdot10^{-2}}_{\pm4.6\cdot10^{-5}}$&
          ${2.21\cdot10^{-2}}_{\pm1.7\cdot10^{-5}}$&
          ${1.20\cdot10^{-2}}_{\pm2.4\cdot10^{-5}}$
         \\ \cline{2-6}
         & \multirow{2}{*}{$25$} & mean &
          ${2.68\cdot10^{-4}}_{\pm3.0\cdot10^{-5}}$&
          ${2.67\cdot10^{-4}}_{\pm4.5\cdot10^{-5}}$&
          ${4.48\cdot10^{-4}}_{\pm4.0\cdot10^{-5}}$&
          ${4.77\cdot10^{-4}}_{\pm7.9\cdot10^{-5}}$&
          ${7.16\cdot10^{-4}}_{\pm1.5\cdot10^{-4}}$
         \\
         & & std. &
          ${3.93\cdot10^{-2}}_{\pm6.0\cdot10^{-6}}$&
          ${3.81\cdot10^{-2}}_{\pm7.0\cdot10^{-6}}$&
          ${3.68\cdot10^{-2}}_{\pm9.7\cdot10^{-6}}$&
          ${2.94\cdot10^{-2}}_{\pm6.9\cdot10^{-6}}$&
          ${1.71\cdot10^{-2}}_{\pm1.1\cdot10^{-5}}$
         \\
         \hline
        \multirow{4}{*}{$\text{RBF}_2$}
         & \multirow{2}{*}{$10$} & mean &
          ${9.76\cdot10^{-4}}_{\pm1.6\cdot10^{-4}}$&
          ${1.23\cdot10^{-3}}_{\pm2.1\cdot10^{-4}}$&
          ${2.04\cdot10^{-3}}_{\pm4.5\cdot10^{-4}}$&
          ${1.88\cdot10^{-3}}_{\pm6.4\cdot10^{-4}}$&
          ${2.25\cdot10^{-3}}_{\pm3.4\cdot10^{-4}}$
         \\
\end{tabular}
    \label{sup:tab:gp-inference-2}
\end{table}

\begin{table}[!ht]
    \renewcommand{\arraystretch}{1.4}
    \tiny
    \centering
    \caption{
        MAE (compared to analytic solution) of mean and standard deviation of GP regression.
        More details are provided in \Cref{sec:exp:gp-inference}.
        Results are averaged over 3 runs.
        $\text{RBF}_{\{1,2\}}$: different prior parameters.
    }
    \begin{tabular}{ccl|ccccc}
        \multicolumn{3}{c}{ }
        & \multicolumn{5}{c}{$N=4$}
        \\
        & $d$ & MAE &
        $\text{IWVI}_\text{SIR}$ & DAIS ($2$) & DAIS ($4$) & DAIS ($8$) & DAIS ($16$)
        \\
        \shline
        \multirow{4}{*}{$\text{RBF}_1$}
         & \multirow{2}{*}{$10$} & mean &
          ${2.44\cdot10^{-4}}_{\pm2.1\cdot10^{-5}}$&${6.90\cdot10^{-3}}_{\pm3.4\cdot10^{-4}}$&${6.45\cdot10^{-3}}_{\pm5.0\cdot10^{-5}}$&${1.50\cdot10^{-2}}_{\pm2.2\cdot10^{-4}}$&${1.65\cdot10^{-2}}_{\pm5.2\cdot10^{-4}}$
         \\
         & & std. &
          ${9.93\cdot10^{-3}}_{\pm1.2\cdot10^{-4}}$&${4.21\cdot10^{-2}}_{\pm2.4\cdot10^{-4}}$&${1.33\cdot10^{-2}}_{\pm1.3\cdot10^{-4}}$&${2.28\cdot10^{-2}}_{\pm1.9\cdot10^{-4}}$&${1.44\cdot10^{-2}}_{\pm1.6\cdot10^{-4}}$
         \\ \cline{2-8}
         & \multirow{2}{*}{$25$} & mean &
          ${8.54\cdot10^{-4}}_{\pm3.6\cdot10^{-5}}$&${5.02\cdot10^{-3}}_{\pm7.3\cdot10^{-4}}$&${6.88\cdot10^{-3}}_{\pm3.7\cdot10^{-4}}$&${8.42\cdot10^{-3}}_{\pm4.0\cdot10^{-4}}$&${6.77\cdot10^{-3}}_{\pm1.2\cdot10^{-4}}$
         \\
         & & std. &
          ${3.24\cdot10^{-2}}_{\pm8.6\cdot10^{-4}}$&${3.91\cdot10^{-2}}_{\pm6.1\cdot10^{-4}}$&${3.12\cdot10^{-2}}_{\pm2.4\cdot10^{-4}}$&${3.43\cdot10^{-2}}_{\pm3.5\cdot10^{-4}}$&${2.79\cdot10^{-2}}_{\pm1.0\cdot10^{-4}}$
         \\
         \hline
        \multirow{4}{*}{$\text{RBF}_2$}
         & \multirow{2}{*}{$10$} & mean &
          ${1.34\cdot10^{-5}}_{\pm4.3\cdot10^{-6}}$&${4.98\cdot10^{-3}}_{\pm2.1\cdot10^{-5}}$&${2.00\cdot10^{-3}}_{\pm5.1\cdot10^{-5}}$&${1.84\cdot10^{-3}}_{\pm1.1\cdot10^{-4}}$&${2.66\cdot10^{-3}}_{\pm7.9\cdot10^{-5}}$
         \\
         & & std. &
          ${1.50\cdot10^{-3}}_{\pm1.5\cdot10^{-4}}$&${3.14\cdot10^{-3}}_{\pm1.2\cdot10^{-4}}$&${1.97\cdot10^{-3}}_{\pm1.4\cdot10^{-4}}$&${1.71\cdot10^{-3}}_{\pm4.1\cdot10^{-5}}$&${1.80\cdot10^{-3}}_{\pm5.7\cdot10^{-5}}$
         \\
\end{tabular}
    \label{sup:tab:gp-inference-2A}
\end{table}

\begin{table}[!ht]
    \renewcommand{\arraystretch}{1.4}
    \tiny
    \centering
    \caption{
        MAE (compared to analytic solution) of mean and standard deviation of GP regression.
        More details are provided in \Cref{sec:exp:gp-inference}.
        Results are averaged over 3 runs.
        $\text{RBF}_{\{1,2\}}$: different prior parameters.
    }
    \begin{tabular}{ccl|ccc}
        \multicolumn{3}{c}{ }
        & \multicolumn{3}{c}{$N=16$}
        \\
        & $d$ & MAE &
        $\text{DAIS}_0$ ($2$) & $\text{DAIS}_0$ ($4$) & $\text{DAIS}_0$ ($8$)
        \\
        \shline
        \multirow{4}{*}{$\text{RBF}_1$}
         & \multirow{2}{*}{$10$} & mean &
          ${7.50\cdot10^{-4}}_{\pm2.1\cdot10^{-4}}$&
          ${1.35\cdot10^{-3}}_{\pm3.4\cdot10^{-4}}$&
          ${1.02\cdot10^{-3}}_{\pm1.5\cdot10^{-4}}$
         \\
         & & std. &
          ${3.64\cdot10^{-2}}_{\pm7.9\cdot10^{-5}}$&
          ${2.56\cdot10^{-2}}_{\pm6.9\cdot10^{-6}}$&
          ${1.17\cdot10^{-2}}_{\pm5.3\cdot10^{-5}}$
         \\ \cline{2-6}
         & \multirow{2}{*}{$25$} & mean &
          ${3.86\cdot10^{-4}}_{\pm5.8\cdot10^{-5}}$&
          ${6.58\cdot10^{-4}}_{\pm5.4\cdot10^{-5}}$&
          ${6.81\cdot10^{-4}}_{\pm9.9\cdot10^{-5}}$
         \\
         & & std. &
          ${3.66\cdot10^{-2}}_{\pm4.5\cdot10^{-6}}$&
          ${3.55\cdot10^{-2}}_{\pm5.7\cdot10^{-5}}$&
          ${2.64\cdot10^{-2}}_{\pm9.0\cdot10^{-6}}$
         \\
         \hline
        \multirow{4}{*}{$\text{RBF}_2$}
         & \multirow{2}{*}{$10$} & mean &
          ${1.94\cdot10^{-3}}_{\pm2.5\cdot10^{-4}}$&
          ${1.80\cdot10^{-3}}_{\pm2.0\cdot10^{-4}}$&
          ${2.88\cdot10^{-3}}_{\pm7.4\cdot10^{-4}}$
         \\
         & & std. &
          ${4.15\cdot10^{-3}}_{\pm1.3\cdot10^{-4}}$&
          ${6.04\cdot10^{-3}}_{\pm5.3\cdot10^{-5}}$&
          ${4.95\cdot10^{-3}}_{\pm6.3\cdot10^{-5}}$
         \\ \cline{2-6}
         & \multirow{2}{*}{$25$} & mean &
          ${6.14\cdot10^{-4}}_{\pm6.0\cdot10^{-5}}$&
          ${8.55\cdot10^{-4}}_{\pm9.8\cdot10^{-5}}$&
          ${1.31\cdot10^{-3}}_{\pm2.5\cdot10^{-4}}$
         \\
         & & std. &
          ${3.48\cdot10^{-2}}_{\pm4.1\cdot10^{-5}}$&
          ${2.90\cdot10^{-2}}_{\pm4.1\cdot10^{-5}}$&
          ${2.20\cdot10^{-2}}_{\pm5.3\cdot10^{-6}}$
    \end{tabular}
    \label{sup:tab:gp-inference-3}
\end{table}

\begin{table}[!ht]
    \renewcommand{\arraystretch}{1.4}
    \tiny
    \centering
    \caption{
        MAE (compared to analytic solution) of mean and standard deviation of GP regression.
        More details are provided in \Cref{sec:exp:gp-inference}.
        Results are averaged over 3 runs.
        $\text{RBF}_{\{1,2\}}$: different prior parameters.
    }
    \begin{tabular}{ccl|ccc}
        \multicolumn{3}{c}{ }
        & \multicolumn{3}{c}{$N=16$}
        \\
        & $d$ & MAE &
        DAIS ($2$) & DAIS ($4$) & DAIS ($8$)
        \\
        \shline
        \multirow{4}{*}{$\text{RBF}_1$}
         & \multirow{2}{*}{$10$} & mean &
          ${1.15\cdot10^{-2}}_{\pm1.2\cdot10^{-3}}$&${7.61\cdot10^{-3}}_{\pm7.1\cdot10^{-4}}$&${1.60\cdot10^{-2}}_{\pm4.0\cdot10^{-4}}$
         \\
         & & std. &
          ${4.06\cdot10^{-2}}_{\pm3.0\cdot10^{-4}}$&${9.66\cdot10^{-3}}_{\pm2.3\cdot10^{-4}}$&${1.82\cdot10^{-2}}_{\pm4.0\cdot10^{-4}}$
         \\ \cline{2-6}
         & \multirow{2}{*}{$25$} & mean &
          ${4.92\cdot10^{-3}}_{\pm7.2\cdot10^{-4}}$&${6.26\cdot10^{-3}}_{\pm9.3\cdot10^{-4}}$&${7.93\cdot10^{-3}}_{\pm4.8\cdot10^{-4}}$
         \\
         & & std. &
          ${3.77\cdot10^{-2}}_{\pm3.7\cdot10^{-4}}$&${3.15\cdot10^{-2}}_{\pm1.0\cdot10^{-3}}$&${3.08\cdot10^{-2}}_{\pm3.4\cdot10^{-4}}$
         \\
         \hline
        \multirow{4}{*}{$\text{RBF}_2$}
         & \multirow{2}{*}{$10$} & mean &
          ${4.55\cdot10^{-3}}_{\pm2.6\cdot10^{-4}}$&${2.46\cdot10^{-3}}_{\pm1.3\cdot10^{-4}}$&${2.38\cdot10^{-3}}_{\pm4.0\cdot10^{-4}}$
         \\
         & & std. &
          ${3.29\cdot10^{-3}}_{\pm1.4\cdot10^{-4}}$&${2.33\cdot10^{-3}}_{\pm2.6\cdot10^{-5}}$&${2.09\cdot10^{-3}}_{\pm7.3\cdot10^{-4}}$
         \\ \cline{2-6}
         & \multirow{2}{*}{$25$} & mean &
          ${6.01\cdot10^{-3}}_{\pm4.5\cdot10^{-4}}$&${1.90\cdot10^{-2}}_{\pm1.7\cdot10^{-3}}$&${7.22\cdot10^{-3}}_{\pm7.7\cdot10^{-4}}$
         \\
         & & std. &
          ${3.88\cdot10^{-2}}_{\pm8.5\cdot10^{-4}}$&${2.73\cdot10^{-2}}_{\pm1.8\cdot10^{-3}}$&${5.17\cdot10^{-3}}_{\pm1.7\cdot10^{-4}}$
    \end{tabular}
    \label{sup:tab:gp-inference-3A}
\end{table}

\begin{table}[ht]
    \renewcommand{\arraystretch}{1.4}
    \tiny
    \centering
    \caption{
        MAE (compared to analytic solution) of mean and standard deviation of GP regression.
        More details are provided in \Cref{sec:exp:gp-inference}.
        Results are averaged over 3 runs.
        $\text{RBF}_{\{1,2\}}$: different prior parameters.
    }
    \begin{tabular}{ccl|cccc}
        & $d$ & MAE &
        $\text{MSC}_{1\text{c}}$ ($N=1$) & 
        $\text{MSC}_{1\text{c}}$ ($N=2$) &
        $\text{MSC}_{1\text{c}}$ ($N=4$) &
        $\text{MSC}_{1\text{c}}$ ($N=16$)
        \\
        \shline
        \multirow{4}{*}{$\text{RBF}_1$}
         & \multirow{2}{*}{$10$} & mean &
          ${9.22\cdot10^{-1}}_{\pm0.0}$&
          ${4.10\cdot10^{-2}}_{\pm7.3\cdot10^{-3}}$&
          ${3.11\cdot10^{-2}}_{\pm7.6\cdot10^{-3}}$&
          ${1.87\cdot10^{-2}}_{\pm4.5\cdot10^{-3}}$
         \\
         & & std. &
          ${6.66\cdot10^{-2}}_{\pm0.0}$&
          ${5.02\cdot10^{-2}}_{\pm5.2\cdot10^{-4}}$&
          ${4.68\cdot10^{-2}}_{\pm3.0\cdot10^{-4}}$&
          ${4.25\cdot10^{-2}}_{\pm1.5\cdot10^{-4}}$
         \\ \cline{2-7}
         & \multirow{2}{*}{$25$} & mean &
          ${9.20\cdot10^{-1}}_{\pm0.0}$&
          ${2.72\cdot10^{-1}}_{\pm4.6\cdot10^{-2}}$&
          ${5.02\cdot10^{-2}}_{\pm6.4\cdot10^{-3}}$&
          ${2.72\cdot10^{-2}}_{\pm0.77\cdot10^{-3}}$
         \\
         & & std. &
          ${4.54\cdot10^{-2}}_{\pm0.0}$&
          ${4.20\cdot10^{-2}}_{\pm1.6\cdot10^{-4}}$&
          ${4.02\cdot10^{-2}}_{\pm2.7\cdot10^{-5}}$&
          ${3.89\cdot10^{-2}}_{\pm1.5\cdot10^{-4}}$
         \\
         \hline
        \multirow{4}{*}{$\text{RBF}_2$}
         & \multirow{2}{*}{$10$} & mean &
          ${6.23\cdot10^{-1}}_{\pm0.0}$&
          ${1.72\cdot10^{-2}}_{\pm2.6\cdot10^{-3}}$&
          ${7.30\cdot10^{-3}}_{\pm1.6\cdot10^{-3}}$&
          ${4.16\cdot10^{-3}}_{\pm2.9\cdot10^{-4}}$
         \\
         & & std. &
          ${8.80\cdot10^{-2}}_{\pm0.0}$&
          ${7.70\cdot10^{-3}}_{\pm5.2\cdot10^{-4}}$&
          ${3.57\cdot10^{-3}}_{\pm3.3\cdot10^{-4}}$&
          ${1.67\cdot10^{-3}}_{\pm5.7\cdot10^{-5}}$
         \\
\end{tabular}
    \label{sup:tab:gp-inference-4}
\end{table}

\begin{table}[ht]
    \renewcommand{\arraystretch}{1.4}
    \tiny
    \centering
    \caption{
        MAE (compared to analytic solution) of mean and standard deviation of GP regression for MSC with 8 parallel chains ($\text{MSC}_{8\text{c}}$).
        Results are averaged over 3 runs.
        $\text{RBF}_{\{1,2\}}$: different prior parameters.
    }
    \begin{tabular}{ccl|cccc}
        & $d$ & MAE &
        $\text{MSC}_{8\text{c}}$ ($N=1$)
        & $\text{MSC}_{8\text{c}}$ ($N=2$)
        & $\text{MSC}_{8\text{c}}$ ($N=4$)
        & $\text{MSC}_{8\text{c}}$ ($N=16$)
        \\
        \shline
        \multirow{4}{*}{$\text{RBF}_1$}
         & \multirow{2}{*}{$10$} & mean &
          ${9.22\cdot10^{-1}}_{\pm0.0}$&
          ${1.66\cdot10^{-2}}_{\pm3.6\cdot10^{-3}}$&
          ${1.66\cdot10^{-2}}_{\pm5.2\cdot10^{-3}}$&
          ${1.74\cdot10^{-2}}_{\pm2.4\cdot10^{-3}}$
         \\
         & & std. &
          ${6.66\cdot10^{-2}}_{\pm0.0}$&
          ${4.57\cdot10^{-2}}_{\pm1.3\cdot10^{-3}}$&
          ${4.29\cdot10^{-2}}_{\pm9.6\cdot10^{-4}}$&
          ${3.66\cdot10^{-2}}_{\pm1.4\cdot10^{-3}}$
         \\ \cline{2-7}
         & \multirow{2}{*}{$25$} & mean &
          ${9.20\cdot10^{-1}}_{\pm0.0}$&
          ${4.17\cdot10^{-2}}_{\pm6.3\cdot10^{-3}}$&
          ${2.03\cdot10^{-2}}_{\pm3.7\cdot10^{-3}}$&
          ${1.14\cdot10^{-2}}_{\pm1.6\cdot10^{-3}}$
         \\
         & & std. &
          ${4.54\cdot10^{-2}}_{\pm0.0}$&
          ${3.95\cdot10^{-2}}_{\pm1.0\cdot10^{-4}}$&
          ${3.89\cdot10^{-2}}_{\pm1.4\cdot10^{-4}}$&
          ${3.77\cdot10^{-2}}_{\pm1.5\cdot10^{-4}}$
         \\
         \hline
        \multirow{4}{*}{$\text{RBF}_2$}
         & \multirow{2}{*}{$10$} & mean &
          ${6.23\cdot10^{-1}}_{\pm0.0}$&
          ${7.54\cdot10^{-3}}_{\pm2.5\cdot10^{-3}}$&
          ${4.07\cdot10^{-3}}_{\pm9.2\cdot10^{-4}}$&
          ${1.57\cdot10^{-3}}_{\pm2.7\cdot10^{-4}}$
         \\
         & & std. &
          ${8.80\cdot10^{-2}}_{\pm0.0}$&
          ${4.04\cdot10^{-3}}_{\pm7.7\cdot10^{-4}}$&
          ${2.34\cdot10^{-3}}_{\pm6.6\cdot10^{-4}}$&
          ${7.85\cdot10^{-4}}_{\pm2.4\cdot10^{-4}}$
         \\
\end{tabular}
    \label{sup:tab:gp-inference-4-A}
\end{table}

\clearpage
\section{Bayesian Logistic Regression} \label{sup:sec:log-reg}

We provide further details on model, the sampling procedure, and quantitative results of the logistic regression experiments on the five datasets considered.
The experiment is discussed in \Cref{sec:exp:log-reg} of the main text.
We use an Intel XEON CPU E5-2650 v4 for running the small-scale experiments and a single NVIDIA GeForce GTX 1080 Ti for the large-scale experiments.

\subsection{Model}

We train our models for $100,000$ iterations on full-batch gradients.
We use the Adam optimizer with a learning rate of $10^{-3}$.
DAIS utilizes the parameterization of \citet{zhang2021differentiable}.
In addition to the parameters of $q_0$, we learn the diagonal $\vd$ of the mass matrix $\mM = \mathrm{diag}(\vd)$, the annealing schedule ($\beta_1, \dots, \beta_{K-1}$), and the step widths of the sampler.

\subsection{Sampling}

We sample the model described in the main text with HMC and compare the learned variational distributions of VI, IWVI, and $\text{DAIS}_0$ to the samples (that we treat as ``ground truth'').
We use the leapfrog integrator with an identity mass matrix for $n_l$ steps and a step size of $\epsilon_{\text{HMC}}$, $n_b$ burn-in steps, take every $n_e$-th sample and sample $n_t$ in total.
We apply MH correction steps. 
\Cref{sup:tab:hmc-sampling} shows the corresponding hyperparameters.

\begin{table}[ht]
    \renewcommand{\arraystretch}{1.4}
    \small
    \centering
    \caption{
        Hyperparameters of HMC sampling for the \textit{sonar} dataset and the \textit{ionosphere} dataset.
    }
    \begin{tabular}{l|ccccc}
        & $\epsilon_\text{HMC}$
        & $n_l$
        & $n_b$
        & $n_e$
        & $n_t$
        \\
        \shline
        \textit{sonar}
        & $0.001$
        & $50$
        & $10,000$
        & $10$
        & $10,000$
        \\
        \textit{ionosphere}
        & $0.001$
        & $50$
        & $10,000$
        & $5$
        & $10,000$
    \end{tabular}
    \label{sup:tab:hmc-sampling}
\end{table}

\subsection{Complementary Quantitative Results}

\Cref{sup:tab:log-reg-0}, \Cref{sup:tab:log-reg-0A}, \Cref{sup:tab:log-reg-1}, \Cref{sup:tab:log-reg-1A}, \Cref{sup:tab:log-reg-2}, \Cref{sup:tab:log-reg-2A}, \Cref{sup:tab:log-reg-3}, \Cref{sup:tab:log-reg-3A}, and \Cref{sup:tab:log-reg-4} provide additional results on the experiment.

\begin{table}[ht]
    \renewcommand{\arraystretch}{1.4}
    \tiny
    \centering
    \caption{
        MAE calculated between learned mean and standard deviation to HMC samples for Bayesian logistic regression.
        Details can be found in \Cref{sec:exp:log-reg}.
        Results are averaged over $3$ runs.
    }
\begin{tabular}{lcl|ccccc}
        \multicolumn{3}{c}{ } 
        & \multicolumn{5}{c}{$N=1$} 
        \\
        & $d$ & MAE
        & IWVI & $\text{DAIS}_0$ ($2$) & $\text{DAIS}_0$ ($4$) &
        $\text{DAIS}_0$ ($8$) & $\text{DAIS}_0$ ($16$)
        \\
        \shline
        \multirow{2}{*}{\!\!\!\!\!\textit{sonar}}
         & \multirow{2}{*}{\!\!\!\textit{$61$}} & mean &
          ${9.00\cdot10^{-2}}_{\pm1.0\cdot10^{-3}}$&
          ${8.67\cdot10^{-2}}_{\pm2.2\cdot10^{-3}}$&
          ${8.76\cdot10^{-2}}_{\pm2.9\cdot10^{-3}}$&
          ${8.69\cdot10^{-2}}_{\pm1.8\cdot10^{-3}}$&
          ${8.68\cdot10^{-2}}_{\pm9.7\cdot10^{-4}}$
         \\
         && std. &
          ${2.86\cdot10^{-1}}_{\pm1.9\cdot10^{-3}}$&
          ${1.77\cdot10^{-1}}_{\pm9.5\cdot10^{-4}}$&
          ${1.23\cdot10^{-1}}_{\pm1.5\cdot10^{-3}}$&
          ${1.22\cdot10^{-1}}_{\pm1.4\cdot10^{-3}}$&
          ${1.20\cdot10^{-1}}_{\pm4.3\cdot10^{-4}}$
         \\
         \hline
        \multirow{2}{*}{\!\!\!\!\!\textit{ionosphere}}
         & \multirow{2}{*}{\!\!\!
         \textit{$35$}} & mean &
          ${9.23\cdot10^{-2}}_{\pm5.8\cdot10^{-4}}$&
          ${8.66\cdot10^{-2}}_{\pm5.0\cdot10^{-4}}$&
          ${7.87\cdot10^{-2}}_{\pm3.5\cdot10^{-3}}$&
          ${8.38\cdot10^{-2}}_{\pm2.2\cdot10^{-3}}$&
          ${8.57\cdot10^{-2}}_{\pm2.9\cdot10^{-3}}$
         \\
         && std. &
          ${2.32\cdot10^{-1}}_{\pm8.5\cdot10^{-4}}$&
          ${2.01\cdot10^{-1}}_{\pm1.1\cdot10^{-3}}$&
          ${1.57\cdot10^{-1}}_{\pm6.5\cdot10^{-3}}$&
          ${1.20\cdot10^{-1}}_{\pm1.3\cdot10^{-3}}$&
          ${8.40\cdot10^{-2}}_{\pm1.3\cdot10^{-3}}$
        \\
\end{tabular}
    \label{sup:tab:log-reg-0}
\end{table}

\begin{table}[ht]
    \renewcommand{\arraystretch}{1.4}
    \tiny
    \centering
    \caption{
        MAE calculated between learned mean and standard deviation to HMC samples for Bayesian logistic regression.
        Details can be found in \Cref{sec:exp:log-reg}.
        Results are averaged over $3$ runs.
    }
    \begin{tabular}{lcl|ccccc}
        \multicolumn{3}{c}{ } 
        & \multicolumn{5}{c}{$N=1$} 
        \\
        & $d$ & MAE
        & $\text{IWVI}_\text{SIR}$ & DAIS ($2$) & DAIS ($4$) &
        DAIS ($8$) & DAIS ($16$)
        \\
        \shline
        \multirow{2}{*}{\!\!\!\!\!\textit{sonar}}
         & \multirow{2}{*}{\!\!\!\textit{$61$}} & mean &
          ${9.28\cdot10^{-2}}_{\pm7.6\cdot10^{-4}}$&${3.77\cdot10^{-1}}_{\pm2.2\cdot10^{-2}}$&${2.06\cdot10^{-1}}_{\pm1.4\cdot10^{-2}}$&${1.77\cdot10^{-1}}_{\pm1.5\cdot10^{-2}}$&${1.68\cdot10^{-1}}_{\pm8.2\cdot10^{-3}}$
         \\
         && std. &
          ${2.90\cdot10^{-1}}_{\pm2.1\cdot10^{-3}}$&${3.38\cdot10^{-1}}_{\pm2.1\cdot10^{-2}}$&${1.40\cdot10^{-1}}_{\pm7.1\cdot10^{-3}}$&${1.25\cdot10^{-1}}_{\pm7.9\cdot10^{-3}}$&${1.02\cdot10^{-1}}_{\pm8.1\cdot10^{-3}}$
         \\
         \hline
        \multirow{2}{*}{\!\!\!\!\!\textit{ionosphere}}
         & \multirow{2}{*}{\!\!\!
         \textit{$35$}} & mean &
          ${8.80\cdot10^{-2}}_{\pm5.0\cdot10^{-4}}$&${9.75\cdot10^{-2}}_{\pm3.2\cdot10^{-3}}$&${1.08\cdot10^{-1}}_{\pm1.9\cdot10^{-2}}$&${1.11\cdot10^{-1}}_{\pm2.4\cdot10^{-2}}$&${8.90\cdot10^{-2}}_{\pm2.3\cdot10^{-3}}$
         \\
         && std. &
          ${2.32\cdot10^{-1}}_{\pm8.7\cdot10^{-4}}$&${1.55\cdot10^{-1}}_{\pm8.3\cdot10^{-4}}$&${1.13\cdot10^{-1}}_{\pm1.1\cdot10^{-2}}$&${8.57\cdot10^{-2}}_{\pm9.8\cdot10^{-3}}$&${4.32\cdot10^{-2}}_{\pm3.8\cdot10^{-3}}$
        \\
         \hline
        \multirow{2}{*}{\!\!\!\!\!\textit{heart-disease}}
         & \multirow{2}{*}{\!\!\!
         \textit{$16$}} & mean &
          ${1.82\cdot10^{-2}}_{\pm1.8\cdot10^{-4}}$&${1.68\cdot10^{-2}}_{\pm1.2\cdot10^{-3}}$&${1.84\cdot10^{-2}}_{\pm1.4\cdot10^{-3}}$&${1.85\cdot10^{-2}}_{\pm7.7\cdot10^{-4}}$&${1.62\cdot10^{-2}}_{\pm7.8\cdot10^{-4}}$
         \\
\end{tabular}
    \label{sup:tab:log-reg-0A}
\end{table}

\begin{table}[ht]
    \renewcommand{\arraystretch}{1.4}
    \tiny
    \centering
    \caption{
        MAE calculated between learned mean and standard deviation to HMC samples for Bayesian logistic regression.
        Details can be found in \Cref{sec:exp:log-reg}.
        Results are averaged over $3$ runs.
    }
    \begin{tabular}{lcl|ccccc}
        \multicolumn{3}{c}{ } 
        & \multicolumn{5}{c}{$N=2$} 
        \\
        & $d$ & MAE
        & IWVI & $\text{DAIS}_0$ ($2$) & $\text{DAIS}_0$ ($4$) &
        $\text{DAIS}_0$ ($8$) & $\text{DAIS}_0$ ($16$)
        \\
        \shline
        \multirow{2}{*}{\!\!\!\!\!\textit{sonar}}
         & \multirow{2}{*}{\!\!\!\textit{$61$}} & mean &
          ${8.81\cdot10^{-2}}_{\pm8.7\cdot10^{-4}}$&
          ${8.53\cdot10^{-2}}_{\pm1.3\cdot10^{-3}}$&
          ${8.65\cdot10^{-2}}_{\pm3.5\cdot10^{-4}}$&
          ${8.51\cdot10^{-2}}_{\pm2.2\cdot10^{-3}}$&
          ${8.49\cdot10^{-2}}_{\pm1.5\cdot10^{-3}}$
         \\
         && std. &
          ${2.09\cdot10^{-1}}_{\pm1.2\cdot10^{-3}}$&
          ${1.31\cdot10^{-1}}_{\pm1.2\cdot10^{-3}}$&
          ${9.08\cdot10^{-2}}_{\pm1.3\cdot10^{-3}}$&
          ${9.08\cdot10^{-2}}_{\pm1.3\cdot10^{-3}}$&
          ${8.86\cdot10^{-2}}_{\pm8.2\cdot10^{-4}}$
         \\
         \hline
        \multirow{2}{*}{\!\!\!\!\!\textit{ionosphere}}
         & \multirow{2}{*}{\!\!\!
         \textit{$35$}} & mean &
          ${4.84\cdot10^{-2}}_{\pm3.3\cdot10^{-4}}$&
          ${4.59\cdot10^{-2}}_{\pm3.2\cdot10^{-4}}$&
          ${4.46\cdot10^{-2}}_{\pm5.5\cdot10^{-4}}$&
          ${4.69\cdot10^{-2}}_{\pm4.7\cdot10^{-4}}$&
          ${4.87\cdot10^{-2}}_{\pm5.9\cdot10^{-4}}$
         \\
         && std. &
          ${1.04\cdot10^{-1}}_{\pm3.6\cdot10^{-4}}$&
          ${9.00\cdot10^{-2}}_{\pm2.1\cdot10^{-4}}$&
          ${7.53\cdot10^{-2}}_{\pm1.6\cdot10^{-4}}$&
          ${5.94\cdot10^{-2}}_{\pm8.2\cdot10^{-4}}$&
          ${4.14\cdot10^{-2}}_{\pm1.7\cdot10^{-3}}$
        \\
\end{tabular}
    \label{sup:tab:log-reg-1}
\end{table}

\begin{table}[ht]
    \renewcommand{\arraystretch}{1.4}
    \tiny
    \centering
    \caption{
        MAE calculated between learned mean and standard deviation to HMC samples for Bayesian logistic regression.
        Details can be found in \Cref{sec:exp:log-reg}.
        Results are averaged over $3$ runs.
    }
    \begin{tabular}{lcl|ccccc}
        \multicolumn{3}{c}{ } 
        & \multicolumn{5}{c}{$N=2$} 
        \\
        & $d$ & MAE
        & $\text{IWVI}_\text{SIR}$ & DAIS ($2$) & DAIS ($4$) &
        DAIS ($8$) & DAIS ($16$)
        \\
        \shline
        \multirow{2}{*}{\!\!\!\!\!\textit{sonar}}
         & \multirow{2}{*}{\!\!\!\textit{$61$}} & mean &
          ${9.03\cdot10^{-2}}_{\pm1.4\cdot10^{-3}}$&${3.53\cdot10^{-1}}_{\pm1.8\cdot10^{-2}}$&${1.85\cdot10^{-1}}_{\pm2.0\cdot10^{-2}}$&${1.35\cdot10^{-1}}_{\pm1.1\cdot10^{-2}}$&${1.58\cdot10^{-1}}_{\pm7.7\cdot10^{-3}}$
         \\
         && std. &
          ${2.13\cdot10^{-1}}_{\pm1.0\cdot10^{-3}}$&${2.59\cdot10^{-1}}_{\pm1.6\cdot10^{-2}}$&${1.17\cdot10^{-1}}_{\pm1.3\cdot10^{-2}}$&${8.64\cdot10^{-2}}_{\pm3.2\cdot10^{-3}}$&${9.04\cdot10^{-2}}_{\pm7.9\cdot10^{-3}}$
         \\
         \hline
        \multirow{2}{*}{\!\!\!\!\!\textit{ionosphere}}
         & \multirow{2}{*}{\!\!\!
         \textit{$35$}} & mean &
          ${7.94\cdot10^{-2}}_{\pm9.8\cdot10^{-4}}$&${9.14\cdot10^{-2}}_{\pm6.5\cdot10^{-3}}$&${1.39\cdot10^{-1}}_{\pm3.1\cdot10^{-2}}$&${1.19\cdot10^{-1}}_{\pm3.2\cdot10^{-3}}$&${9.99\cdot10^{-2}}_{\pm1.7\cdot10^{-2}}$
         \\
         && std. &
          ${1.81\cdot10^{-1}}_{\pm6.5\cdot10^{-4}}$&${1.20\cdot10^{-1}}_{\pm5.8\cdot10^{-3}}$&${1.26\cdot10^{-1}}_{\pm2.5\cdot10^{-2}}$&${7.04\cdot10^{-2}}_{\pm3.7\cdot10^{-3}}$&${6.06\cdot10^{-2}}_{\pm3.9\cdot10^{-3}}$
        \\
         \hline
        \multirow{2}{*}{\!\!\!\!\!\textit{heart-disease}}
         & \multirow{2}{*}{\!\!\!
         \textit{$16$}} & mean &
          ${1.63\cdot10^{-2}}_{\pm1.4\cdot10^{-3}}$&${1.84\cdot10^{-2}}_{\pm1.7\cdot10^{-3}}$&${1.84\cdot10^{-2}}_{\pm5.8\cdot10^{-4}}$&${1.86\cdot10^{-2}}_{\pm1.5\cdot10^{-3}}$&${1.98\cdot10^{-2}}_{\pm4.7\cdot10^{-3}}$
         \\
\end{tabular}
    \label{sup:tab:log-reg-1A}
\end{table}

\begin{table}[ht]
    \renewcommand{\arraystretch}{1.4}
    \tiny
    \centering
    \caption{
        MAE calculated between learned mean and standard deviation to HMC samples for Bayesian logistic regression.
        Details can be found in \Cref{sec:exp:log-reg}.
        Results are averaged over $3$ runs.
    }
    \begin{tabular}{lcl|ccccc}
        \multicolumn{3}{c}{ } 
        & \multicolumn{5}{c}{$N=4$} 
        \\
        & $d$ & MAE
        & IWVI & $\text{DAIS}_0$ ($2$) & $\text{DAIS}_0$ ($4$) & $\text{DAIS}_0$ ($8$) & $\text{DAIS}_0$ ($16$)
        \\
        \shline
        \multirow{2}{*}{\textit{sonar}}
         & \multirow{2}{*}{\textit{$61$}} & mean &
          ${8.68\cdot10^{-2}}_{\pm1.1\cdot10^{-3}}$&
          ${8.45\cdot10^{-2}}_{\pm8.5\cdot10^{-4}}$&
          ${8.64\cdot10^{-2}}_{\pm9.0\cdot10^{-4}}$&
          ${8.89\cdot10^{-2}}_{\pm2.3\cdot10^{-3}}$&
          ${8.42\cdot10^{-2}}_{\pm1.1\cdot10^{-3}}$
         \\
         && std. &
          ${1.48\cdot10^{-1}}_{\pm1.1\cdot10^{-3}}$&
          ${9.98\cdot10^{-2}}_{\pm1.5\cdot10^{-3}}$&
          ${6.92\cdot10^{-2}}_{\pm1.9\cdot10^{-3}}$&
          ${6.95\cdot10^{-2}}_{\pm1.0\cdot10^{-3}}$&
          ${6.63\cdot10^{-2}}_{\pm1.8\cdot10^{-3}}$
         \\
         \hline
        \multirow{2}{*}{\textit{ionosphere}}
         & \multirow{2}{*}{\textit{$35$}} & mean &
          ${4.61\cdot10^{-2}}_{\pm8.7\cdot10^{-4}}$&
          ${4.48\cdot10^{-2}}_{\pm4.0\cdot10^{-4}}$&
          ${4.41\cdot10^{-2}}_{\pm1.1\cdot10^{-3}}$&
          ${4.42\cdot10^{-2}}_{\pm1.8\cdot10^{-3}}$&
          ${4.59\cdot10^{-2}}_{\pm9.7\cdot10^{-4}}$
         \\
         && std. &
          ${7.97\cdot10^{-2}}_{\pm5.7\cdot10^{-4}}$&
          ${6.88\cdot10^{-2}}_{\pm3.2\cdot10^{-4}}$&
          ${5.63\cdot10^{-2}}_{\pm6.3\cdot10^{-4}}$&
          ${5.65\cdot10^{-2}}_{\pm6.7\cdot10^{-4}}$&
          ${4.26\cdot10^{-2}}_{\pm3.0\cdot10^{-3}}$
        \\
\end{tabular}
    \label{sup:tab:log-reg-2}
\end{table}

\begin{table}[ht]
    \renewcommand{\arraystretch}{1.4}
    \tiny
    \centering
    \caption{
        MAE calculated between learned mean and standard deviation to HMC samples for Bayesian logistic regression.
        Details can be found in \Cref{sec:exp:log-reg}.
        Results are averaged over $3$ runs.
    }
    \begin{tabular}{lcl|ccccc}
        \multicolumn{3}{c}{ } 
        & \multicolumn{5}{c}{$N=4$} 
        \\
        & $d$ & MAE
        & $\text{IWVI}_\text{SIR}$ & DAIS ($2$) & DAIS ($4$) &
        DAIS ($8$) & DAIS ($16$)
        \\
        \shline
        \multirow{2}{*}{\!\!\!\!\!\textit{sonar}}
         & \multirow{2}{*}{\!\!\!\textit{$61$}} & mean &
          ${8.86\cdot10^{-2}}_{\pm1.2\cdot10^{-3}}$&${2.45\cdot10^{-1}}_{\pm4.3\cdot10^{-2}}$&${1.54\cdot10^{-1}}_{\pm1.6\cdot10^{-2}}$&${1.21\cdot10^{-1}}_{\pm6.7\cdot10^{-3}}$&${1.40\cdot10^{-1}}_{\pm9.4\cdot10^{-3}}$
         \\
         && std. &
          ${1.52\cdot10^{-1}}_{\pm1.1\cdot10^{-3}}$&${1.76\cdot10^{-1}}_{\pm2.8\cdot10^{-2}}$&${9.18\cdot10^{-2}}_{\pm1.0\cdot10^{-2}}$&${7.43\cdot10^{-2}}_{\pm6.3\cdot10^{-3}}$&${7.80\cdot10^{-2}}_{\pm3.1\cdot10^{-3}}$
         \\
         \hline
        \multirow{2}{*}{\!\!\!\!\!\textit{ionosphere}}
         & \multirow{2}{*}{\!\!\!
         \textit{$35$}} & mean &
          ${7.67\cdot10^{-2}}_{\pm7.0\cdot10^{-4}}$&${1.29\cdot10^{-1}}_{\pm3.8\cdot10^{-2}}$&${2.27\cdot10^{-1}}_{\pm2.3\cdot10^{-2}}$&${1.01\cdot10^{-1}}_{\pm5.4\cdot10^{-3}}$&${8.49\cdot10^{-2}}_{\pm2.8\cdot10^{-3}}$
         \\
         && std. &
          ${1.39\cdot10^{-1}}_{\pm9.9\cdot10^{-4}}$&${1.25\cdot10^{-1}}_{\pm3.0\cdot10^{-2}}$&${1.68\cdot10^{-1}}_{\pm1.9\cdot10^{-2}}$&${7.11\cdot10^{-2}}_{\pm1.7\cdot10^{-3}}$&${5.82\cdot10^{-2}}_{\pm4.7\cdot10^{-3}}$
        \\
         \hline
        \multirow{2}{*}{\!\!\!\!\!\textit{heart-disease}}
         & \multirow{2}{*}{\!\!\!
         \textit{$16$}} & mean &
          ${1.70\cdot10^{-2}}_{\pm1.4\cdot10^{-4}}$&${1.63\cdot10^{-2}}_{\pm1.3\cdot10^{-3}}$&${2.20\cdot10^{-2}}_{\pm2.2\cdot10^{-3}}$&${1.82\cdot10^{-2}}_{\pm3.7\cdot10^{-4}}$&${1.98\cdot10^{-2}}_{\pm2.2\cdot10^{-3}}$
         \\
\end{tabular}
    \label{sup:tab:log-reg-2A}
\end{table}

\begin{table}[ht]
    \renewcommand{\arraystretch}{1.4}
    \tiny
    \centering
    \caption{
        MAE calculated between learned mean and standard deviation to HMC samples for Bayesian logistic regression.
        Details can be found in \Cref{sec:exp:log-reg}.
        Results are averaged over $3$ runs.
    }
    \begin{tabular}{lcl|cccccc}
        \multicolumn{3}{c}{ } 
        & \multicolumn{3}{c}{$N=16$} 
        \\
        & $d$ & MAE
        & $\text{DAIS}_0$ ($2$) & $\text{DAIS}_0$ ($4$) & $\text{DAIS}_0$ ($8$)
        \\
        \shline
        \multirow{2}{*}{\textit{sonar}}
         & \multirow{2}{*}{\textit{$61$}} & mean &
          ${8.79\cdot10^{-2}}_{\pm1.9\cdot10^{-3}}$&
          ${8.82\cdot10^{-2}}_{\pm2.0\cdot10^{-3}}$&
          ${8.37\cdot10^{-2}}_{\pm1.3\cdot10^{-3}}$
         \\
         && std. &
          ${6.11\cdot10^{-2}}_{\pm1.0\cdot10^{-3}}$&
          ${4.43\cdot10^{-2}}_{\pm5.0\cdot10^{-4}}$&
          ${4.57\cdot10^{-2}}_{\pm4.8\cdot10^{-4}}$
         \\
         \hline
        \multirow{2}{*}{\textit{ionosphere}}
         & \multirow{2}{*}{\textit{$35$}} & mean &
          ${4.30\cdot10^{-2}}_{\pm1.0\cdot10^{-3}}$&
          ${4.53\cdot10^{-2}}_{\pm8.9\cdot10^{-4}}$&
          ${4.49\cdot10^{-2}}_{\pm1.1\cdot10^{-3}}$
         \\
         && std. &
          ${4.27\cdot10^{-2}}_{\pm2.9\cdot10^{-4}}$&
          ${3.41\cdot10^{-2}}_{\pm6.1\cdot10^{-4}}$&
          ${3.46\cdot10^{-2}}_{\pm4.5\cdot10^{-4}}$
        \\
         \hline
        \multirow{2}{*}{\textit{heart-disease}}
         & \multirow{2}{*}{\textit{$16$}} & mean &
          ${2.03\cdot10^{-2}}_{\pm2.1\cdot10^{-3}}$&
          ${1.94\cdot10^{-2}}_{\pm3.0\cdot10^{-3}}$&
          ${2.21\cdot10^{-2}}_{\pm4.6\cdot10^{-3}}$
         \\
         && std. &
          ${1.41\cdot10^{-2}}_{\pm1.8\cdot10^{-3}}$&
          ${9.70\cdot10^{-3}}_{\pm1.4\cdot10^{-3}}$&
          ${7.85\cdot10^{-3}}_{\pm4.2\cdot10^{-4}}$
        \\
\end{tabular}
    \label{sup:tab:log-reg-3}
\end{table}

\begin{table}[ht]
    \renewcommand{\arraystretch}{1.4}
    \tiny
    \centering
    \caption{
        MAE calculated between learned mean and standard deviation to HMC samples for Bayesian logistic regression.
        Details can be found in \Cref{sec:exp:log-reg}.
        Results are averaged over $3$ runs.
    }
    \begin{tabular}{lcl|cccccc}
        \multicolumn{3}{c}{ } 
        & \multicolumn{3}{c}{$N=16$} 
        \\
        & $d$ & MAE
        & DAIS ($2$) & DAIS ($4$) & DAIS ($8$)
        \\
        \shline
        \multirow{2}{*}{\textit{sonar}}
         & \multirow{2}{*}{\textit{$61$}} & mean &
          ${1.87\cdot10^{-1}}_{\pm6.0\cdot10^{-2}}$&${1.20\cdot10^{-1}}_{\pm7.9\cdot10^{-3}}$&${1.07\cdot10^{-1}}_{\pm5.7\cdot10^{-3}}$
         \\
         && std. &
          ${1.25\cdot10^{-1}}_{\pm2.5\cdot10^{-2}}$&${7.26\cdot10^{-2}}_{\pm6.0\cdot10^{-3}}$&${6.06\cdot10^{-2}}_{\pm5.0\cdot10^{-3}}$
         \\
         \hline
        \multirow{2}{*}{\textit{ionosphere}}
         & \multirow{2}{*}{\textit{$35$}} & mean &
          ${1.36\cdot10^{-1}}_{\pm4.3\cdot10^{-2}}$&${2.05\cdot10^{-1}}_{\pm5.2\cdot10^{-2}}$&${1.15\cdot10^{-1}}_{\pm9.7\cdot10^{-3}}$
         \\
         && std. &
          ${1.08\cdot10^{-1}}_{\pm2.1\cdot10^{-2}}$&${1.27\cdot10^{-1}}_{\pm2.7\cdot10^{-2}}$&${7.53\cdot10^{-2}}_{\pm2.0\cdot10^{-3}}$
        \\
         \hline
        \multirow{2}{*}{\textit{heart-disease}}
         & \multirow{2}{*}{\textit{$16$}} & mean &
          ${1.99\cdot10^{-2}}_{\pm2.0\cdot10^{-3}}$&${1.94\cdot10^{-2}}_{\pm3.0\cdot10^{-3}}$&${2.21\cdot10^{-2}}_{\pm4.2\cdot10^{-3}}$
         \\
         && std. &
          ${1.44\cdot10^{-2}}_{\pm1.6\cdot10^{-3}}$&${1.07\cdot10^{-2}}_{\pm1.3\cdot10^{-3}}$&${8.96\cdot10^{-3}}_{\pm1.5\cdot10^{-4}}$
        \\
\end{tabular}
    \label{sup:tab:log-reg-3A}
\end{table}

\begin{table}[ht]
    \renewcommand{\arraystretch}{1.4}
    \tiny
    \centering
    \caption{
        MAE calculated between learned mean and standard deviation to HMC samples for Bayesian logistic regression.
        Details can be found in \Cref{sec:exp:log-reg}.
        Results are averaged over $3$ runs.
    }
    \begin{tabular}{lcl|cccc}
        & $d$ & MAE
        & $\text{MSC}_{1\text{c}}$ ($N=1$)
        & $\text{MSC}_{1\text{c}}$ ($N=4$)
        & $\text{MSC}_{1\text{c}}$ ($N=8$)
        & $\text{MSC}_{1\text{c}}$ ($N=16$)
        \\
        \shline
        \multirow{2}{*}{\textit{sonar}}
         & \multirow{2}{*}{\textit{$61$}} & mean &
          ${4.71\cdot10^{-1}}_{\pm7.1\cdot10^{-3}}$&
          ${2.30\cdot10^{-1}}_{\pm1.9\cdot10^{-2}}$&
          ${1.57\cdot10^{-1}}_{\pm9.5\cdot10^{-3}}$&
          ${1.39\cdot10^{-1}}_{\pm5.9\cdot10^{-3}}$
         \\
         && std. &
          ${8.85\cdot10^{-1}}_{\pm8.7\cdot10^{-7}}$&
          ${2.10\cdot10^{-1}}_{\pm2.1\cdot10^{-2}}$&
          ${1.43\cdot10^{-1}}_{\pm1.4\cdot10^{-3}}$&
          ${9.43\cdot10^{-2}}_{\pm1.1\cdot10^{-2}}$
         \\
         \hline
        \multirow{2}{*}{\textit{ionosphere}}
         & \multirow{2}{*}{\textit{$35$}} & mean &
          ${7.62\cdot10^{-1}}_{\pm1.2\cdot10^{-2}}$&
          ${2.68\cdot10^{-1}}_{\pm2.4\cdot10^{-2}}$&
          ${2.35\cdot10^{-1}}_{\pm1.1\cdot10^{-2}}$&
          ${2.10\cdot10^{-1}}_{\pm9.4\cdot10^{-3}}$
         \\
         && std. &
          ${5.87\cdot10^{-1}}_{\pm4.4\cdot10^{-7}}$&
          ${1.68\cdot10^{-1}}_{\pm1.5\cdot10^{-3}}$&
          ${1.32\cdot10^{-1}}_{\pm1.2\cdot10^{-2}}$&
          ${1.02\cdot10^{-1}}_{\pm1.8\cdot10^{-3}}$
        \\
         \hline
        \multirow{2}{*}{\textit{heart-disease}}
         & \multirow{2}{*}{\textit{$16$}} & mean &
          ${6.72\cdot10^{-1}}_{\pm1.3\cdot10^{-2}}$&
          ${3.63\cdot10^{-1}}_{\pm1.4\cdot10^{-2}}$&
          ${3.50\cdot10^{-1}}_{\pm1.1\cdot10^{-2}}$&
          ${3.40\cdot10^{-1}}_{\pm5.7\cdot10^{-3}}$
         \\
         && std. &
          ${3.47\cdot10^{-1}}_{\pm3.6\cdot10^{-11}}$&
          ${7.21\cdot10^{-2}}_{\pm2.3\cdot10^{-3}}$&
          ${4.44\cdot10^{-2}}_{\pm1.6\cdot10^{-3}}$&
          ${3.45\cdot10^{-2}}_{\pm1.5\cdot10^{-3}}$
        \\
\end{tabular}
    \label{sup:tab:log-reg-4}
\end{table}

\begin{table}[ht]
    \renewcommand{\arraystretch}{1.4}
    \tiny
    \centering
    \caption{
        MAE calculated between learned mean and standard deviation to HMC samples for Bayesian logistic regression with MSC with $8$ parallel chains ($\text{MSC}_{8\text{c}}$).
        Results are averaged over $3$ runs.
    }
    \begin{tabular}{lcl|ccccc}
        & $d$ & MAE
        & $\text{MSC}_{8\text{c}}$ ($N=1$) 
        & $\text{MSC}_{8\text{c}}$ ($N=4$)
        & $\text{MSC}_{8\text{c}}$ ($N=2$)
        & $\text{MSC}_{8\text{c}}$ ($N=16$)
        \\
        \shline
        \multirow{2}{*}{\!\!\!\!\!\textit{sonar}}
         & \multirow{2}{*}{\!\!\!\textit{$61$}} & mean &
          ${4.54\cdot10^{-1}}_{\pm2.3\cdot10^{-3}}$&
          ${1.70\cdot10^{-1}}_{\pm8.2\cdot10^{-3}}$&
          ${1.37\cdot10^{-1}}_{\pm6.9\cdot10^{-3}}$&
          ${1.33\cdot10^{-1}}_{\pm8.1\cdot10^{-3}}$
         \\
         && std. &
          ${7.96\cdot10^{-1}}_{\pm7.3\cdot10^{-4}}$&
          ${1.40\cdot10^{-1}}_{\pm2.9\cdot10^{-3}}$&
          ${1.06\cdot10^{-1}}_{\pm4.7\cdot10^{-3}}$&
          ${7.26\cdot10^{-2}}_{\pm1.6\cdot10^{-3}}$
         \\
         \hline
        \multirow{2}{*}{\!\!\!\!\!\textit{ionosphere}}
         & \multirow{2}{*}{\!\!\!
         \textit{$35$}} & mean &
          ${7.61\cdot10^{-1}}_{\pm9.7\cdot10^{-3}}$&
          ${2.11\cdot10^{-1}}_{\pm1.2\cdot10^{-2}}$&
          ${2.15\cdot10^{-1}}_{\pm4.6\cdot10^{-3}}$&
          ${2.14\cdot10^{-1}}_{\pm5.0\cdot10^{-3}}$
         \\
         && std. &
          ${5.01\cdot10^{-1}}_{\pm2.7\cdot10^{-3}}$&
          ${1.30\cdot10^{-1}}_{\pm1.5\cdot10^{-3}}$&
          ${1.09\cdot10^{-1}}_{\pm4.6\cdot10^{-3}}$&
          ${8.55\cdot10^{-2}}_{\pm2.4\cdot10^{-4}}$
        \\
         \hline
        \multirow{2}{*}{\!\!\!\!\!\textit{heart-disease}}
         & \multirow{2}{*}{\!\!\!
         \textit{$16$}} & mean &
          ${6.73\cdot10^{-1}}_{\pm5.9\cdot10^{-3}}$&
          ${3.43\cdot10^{-1}}_{\pm1.2\cdot10^{-2}}$&
          ${3.46\cdot10^{-1}}_{\pm5.2\cdot10^{-4}}$&
          ${3.37\cdot10^{-1}}_{\pm3.2\cdot10^{-3}}$
         \\
\end{tabular}
    \label{sup:tab:log-reg-5}
\end{table}

\end{document}